\documentclass[journal]{IEEEtran}
\usepackage{graphicx} 
\usepackage{subfigure} 

\usepackage{algorithm}
\usepackage{algorithmic}

\usepackage{amsmath}
\usepackage{amssymb}
\usepackage{amsthm}
\usepackage{tikz,tikz-3dplot}
\usetikzlibrary{arrows}
\usetikzlibrary{calc,fit}
\usepackage{color}

\ifCLASSINFOpdf
\else
\fi

\newtheorem{definition}{Definition}
\newtheorem{lemma}{Lemma}
\newtheorem{theorem}{Theorem}

\newtheorem*{lemma*}{Lemma}
\newtheorem*{theorem*}{Theorem}

\def\0{\mathbf{0}}
\def\a{\mathbf{a}}
\def\b{\mathbf{b}}
\def\x{\mathbf{x}}
\def\y{\mathbf{y}}
\def\u{\mathbf{u}}
\def\v{\mathbf{v}}
\def\w{\mathbf{w}}
\def\A{\mathbf{A}}
\def\cA{\mathcal{A}}
\def\cP{\mathcal{P}}
\def\cJ{\mathcal{J}}
\def\cK{\mathcal{K}}
\def\cD{\mathcal{D}}
\def\cS{\mathcal{S}}
\def\cV{\mathcal{V}}
\def\cW{\mathcal{W}}

\def\transpose{\top} 

\def\spann{\text{span}}
\def\rank{\text{rank}}
\def\conv{\text{conv}}
\def\card{\text{card}}
\def\BP{\text{BP}}
\def\OMP{\text{OMP}}
\def\vol{\text{vol}}

\def\st{\hspace{1em} \mathrm{s.t.} \hspace{0.5em}}
\def\Re{\mathbb{R}}
\def\Sp{\mathbb{S}}

\tikzset{
	box/.style={rectangle, rounded corners=6pt,
		minimum width=50pt, minimum height=20pt, inner sep=6pt,
		draw=black,thick, fill=white}
}

\definecolor{tiffanyblue}{rgb}{0.04, 0.73, 0.71}
\colorlet{pPntColor}{blue} 
\colorlet{dPntColor}{red} 
\colorlet{oPntColor}{orange} 
\colorlet{cColor}{tiffanyblue} 
\colorlet{pRgnColor}{green} 
\colorlet{dRgnColor}{yellow} 

\begin{document}
%
\title{Subspace-Sparse Representation}
%
%
%


\author{Chong~You
	and~Ren\'e~Vidal
	\thanks{Center for Imaging Science, Johns Hopkins University, Baltimore, MD, 21218, USA. e-mail: \{cyou, rvidal\}@cis.jhu.edu.}

	}

\maketitle

\begin{abstract}
Given an overcomplete dictionary $\A$ and a signal $\b$ that is a linear combination of a few \textit{linearly independent} columns of $\A$, classical sparse recovery theory deals with the problem of recovering the unique sparse representation $\x$ such that $\b = \A \x$. It is known that under certain conditions on $\A$, $\x$ can be recovered by the Basis Pursuit (BP) and the Orthogonal Matching Pursuit (OMP) algorithms. In this work, we consider the more general case where $\b$ lies in a low-dimensional subspace spanned by some columns of $\A$, which are possibly \textit{linearly dependent}. In this case, the sparsest solution $\x$ is generally not unique, and we study the problem that the representation $\x$ identifies the subspace, i.e. the nonzero entries of $\x$ correspond to dictionary atoms that are in the subspace. Such a representation $\x$ is called \textit{subspace-sparse}. We present sufficient conditions for guaranteeing subspace-sparse recovery, which have clear geometric interpretations and explain properties of subspace-sparse recovery. We also show that the sufficient conditions can be satisfied under a randomized model. Our results are applicable to the traditional sparse recovery problem and we get conditions for sparse recovery that are less restrictive than the canonical mutual coherent condition. We also use the results to analyze the sparse representation based classification (SRC) method, for which we get conditions to show its correctness.

\end{abstract}


%
\IEEEpeerreviewmaketitle

\section{Introduction}
%
%
%
%
\IEEEPARstart{S}{parsity} has played an important role in the area of signal processing for the past few years. Given an overcomplete dictionary $\A \in \Re ^{D \times J}$, consider the sparse pursuing program: 
\begin{equation}
\min_{\x} \|\x\|_0 \st \b = \A \x,
\label{eq:L0}
\end{equation}
in which $\|\cdot\|_0$ counts the number of nonzero entries. Sparse representation concerns about the uniqueness of the solution and how the solution can be recovered efficiently \cite{Candes:SPM08,Elad:SIAM09,Mairal:FT12}. Since solving \eqref{eq:L0} is generally intractable computationally, it is usually approached by some approximate algorithms such as Orthogonal Matching Pursuit (OMP) \cite{Pati:Asilomar93} and Basis Pursuit (BP) \cite{Chen:SIAM98}. There has also been studies of these algorithms and the results show that if $\A$ is sufficiently \emph{incoherent} \cite{DonohoElad:PNAS03,Tropp:TIT04,Donoho:TIT06} or satisfies the so-called \emph{restricted isometry property} \cite{Candes:TIT05,Candes:CRAS08,Davenport:TIT10,Mo:ACHA11,Mo:TIT12-RIPinOMP,Cai:TIT14}, then the true sparsest solution can be found by these approximate algorithms. 

In this work, we consider an extension of the canonical sparse recovery to the cases where the dictionary $\A$ is \emph{not} necessarily incoherent. Let $\cA = \{\a_j, j \in \cJ\}$ be the set of all columns of $\A$ in problem \eqref{eq:L0}, where $\cJ = \{1, \cdots, J\}$. We consider the case that the dictionary $\cA$ is subspace-structured, i.e., there is a set $\cJ_0 \subsetneq \cJ$ such that $\cA_0 := \{\a_j, j \in \cJ_0\}$ spans a low dimensional subspace, denoted as $\cS_0$. In this case the dictionary is not necessarily incoherent, e.g., two atoms in $\cA_0$ could be arbitrarily close or even be identical. Moreover, for any $\b \in \cS_0$, the solution to \eqref{eq:L0} is generally not unique, since one can get equal sparsity solutions by using any $d_0$ atoms from $\cA_0$, where $d_0 := \dim (\cS_0)$. The goal in this case is not to recover any specific one of these solutions; observe that all of them have the property that they represent $\b$ using atoms only from $\cA_0$, we study whether the solution to \eqref{eq:L0} has such a general property. A solution that satisfies this property is called \emph{subspace-sparse}. Similar to sparse recovery, in the subspace-sparse recovery problem we study whether the approximate algorithms such as OMP and BP give subspace-sparse representations.

The term of subspace-sparse representation is proposed in \cite{Elhamifar:TPAMI13}, and such a representation is also called to be \emph{subspace-preserving} \cite{Vidal:GPCAbook}, or called to satisfy the \emph{subspace-detection property} \cite{Soltanolkotabi:AS13}, or called to have \emph{exact feature selection} \cite{Dyer:JMLR13} in general non-sparse contexts. The concept plays a key role in analyzing subspace-structured data for the tasks of classification \cite{Wright:PAMI09,Elhamifar:CVPR11,Peng:PAMI12} and clustering \cite{Elhamifar:CVPR09,Liu:ICML10,Liu:TPAMI13,Favaro:CVPR11,Vidal:PRL14,Lu:ECCV12,Wang:NIPS13-LRR+SSC, Heckel:arxiv13,Soltanolkotabi:AS14,Park:NIPS14,Li:CVPR15}, with applications to face recognition, motion segmentation, video segmentation, etc. The idea has also intrigued new methods with applications to visual object tracking \cite{Mei:TPAMI11,Zhong:CVPR12}, action recognition \cite{Yang:JAISE09,Castrodad:IJCV12}, subset selection \cite{Elhamifar:CVPR12}, and so on.

Following the initial work of \cite{Elhamifar:CVPR09}, several recent works \cite{Elhamifar:ICASSP10,Elhamifar:TPAMI13,Soltanolkotabi:AS13,Soltanolkotabi:AS14,Wang-Xu:ICML13,Dyer:JMLR13} have studied the subspace-sparse recovery problem in the context of \emph{subspace clustering}, where the task is to cluster a collection of points lying in a union of subspaces. In this case, the problem is solved by first finding a subspace-sparse representation of each point in terms of a dictionary composed of all other points and then applying spectral clustering to these subspace-sparse representations. Notice, however, that these analyses are specific for the correctness of subspace clustering. In this work we study the more general subspace-sparse recovery problem, where the signal to be represented is an arbitrary point in the subspace $\cS_0$, and the goal is to derive conditions on the dictionary under which the OMP and BP algorithms are guaranteed to give subspace-sparse solutions. Based on the analysis, we also obtain new theoretical conditions for classical sparse recovery and sparse representation based classification.

\subsection{Problem formulation and relation with sparse recovery}

Given a dictionary $\cA = \{ \a_j \in \Re^D, j \in \cJ \}$, suppose that there is a partition $\cJ = \cJ_0 \cup \cJ_c$, such that $\cA_0 := \{ \a_j, j \in \cJ_0 \}$ contains points that are in a subspace $\cS_0 := \spann(\cA_0)$ of dimension $d_0 < D$, and $\cA_c := \{ \a_j, j \in \cJ_c \}$ contains points that are not in the subspace $\cS_0$. For an arbitrary point $\b \in \cS_0$, by applying the BP or the OMP algorithm to $\b$ with dictionary $\cA$, we can get a sparse vector $\x$ such that $\b = \A \x  $. The problem of subspace-sparse recovery is to study the conditions on the dictionary $\cA$ under which the representation $\x$ is \emph{subspace-sparse}, i.e. $x_j \ne 0$ only if $j \in \cJ_0$.  We also assume that all atoms in dictionary $\cA$ are normalized to have unit $\ell_2$ norm. 

Classical sparse recovery is a particular case of subspace-sparse recovery. Assume that there is an unknown vector $\x$ that is $s_0$-sparse (i.e. $\x$ has at most $s_0$ nonzero entries), sparse recovery studies the problem of recovering it from the measurement $\b = \A \x$ by algorithms such as BP and OMP. In order for this problem to be well posed, $\x$ needs to be the \emph{unique} sparsest solution, thus the $s_0$ atoms of $\cA$ corresponding to the $s_0$ nonzero entries of $\x$ must be linearly independent. On the other hand, if we assume that the set $\cA_0$ contains $s_0 := \card(\cA_0)$ linearly independent points in the subspace-sparse recovery problem formulation, then the subspace-sparse solution is unique for any $\b \in \cS_0$. In such cases, the conditions for guaranteeing subspace-sparse recovery also guarantees sparse recovery of any $s_0$-sparse vectors.

\subsection{Results and Contributions}

We summarize our major subspace-sparse recovery results, which is discussed in detail in sections \ref{sec:subspace-sparse-deterministic} and \ref{sec:subspace-sparse-randomized}.

Theorems \ref{thm:PRC} and \ref{thm:DRC} introduce, respectively, the principal recovery condition (PRC) and the dual recovery condition (DRC) for \emph{subspace-sparse} recovery. Both of them are conditions on the dictionary $\cA$ under which both OMP and BP give a subspace-sparse solution for every $\b \in \cS_0$. 

The PRC requires that 
\begin{equation}
	\gamma_0 < s(\cA_c, \cS_0),
\end{equation}
where the left hand side, $\gamma_0$, is the \emph{covering radius} of the points $\cA_0$, which is defined as the smallest angle such that any point in the subspace $\cS_0$ is within angle $\gamma_0$ of at least one point in $\cA_0$. \emph{Covering radius} measures how well distributed the atoms $\cA_0$ are in the subspace $\cS_0$, and should be relatively small if the points are equally distributed in all directions within the subspace and not skewed in a certain direction. The right hand side, $s(\cA_c, \cS_0)$, is the minimum angle between any atom in $\cA_c$ and any point in the subspace $\cS_0$. It is large when all pairs of points from the two sets are sufficiently separated. Thus, intuitively, the PRC requires the atoms $\cA_0$ to be sufficiently well spread-out and the atoms $\cA_c$ to be sufficiently away from the subspace $\cS_0$.

The PRC has the drawback that $\cS_0$ on the right hand side contains infinitely many points, making the requirement too strong. We show that a finite subset of the points in $\cS_0$ is sufficient for this purpose, leading to the DRC: 

\begin{equation}
	\gamma_0 < s(\cA_c, \cD_0).
\end{equation}
where $\cD_0$ is a finite subset of the points in the subspace $\cS_0$, which will be defined in Section \ref{sec:def}. The DRC does not require \emph{all} points in subspace $\cS_0$ to be away from the atoms in $\cA_c$, as done by the PRC. Instead, only a finite number of points $\cD_0$ are sufficient for all the points in $\cS_0$. Hence, the DRC is implied by the PRC, thus it gives a stronger result.

In Theorem \ref{thm:randomized}, we show that the DRC can be satisfied under a probabilistic model. Assume that the atoms in $\cA_0$ are independently and uniformly distributed on the unit sphere of subspace $\cS_0$, and atoms in $\cA_c$ are independently and uniformly distributed on the unit sphere of the ambient space $\Re ^D$, then under the condition that $2 \le d_0 \le \sqrt{D/2}$, the DRC is satisfied with a probability $p$ that 1) is an increasing function of $D$, 2) is a decreasing function of $d$ and 3) goes to $100 \%$ as we increase $\card(\cJ_0)$ to infinity while fix $\card(\cJ_c)/ \card(\cJ_0)$. This says that BP and OMP works better for subspace-sparse recovery with low subspace dimension relative to high ambient dimension and for densely sampled dictionary.

\subsection{Applications}

In section \ref{sec:sparse-recovery} we show that our results of subspace-sparse recovery can be applied to the analysis of the traditional sparse recovery problem. The results will be new conditions on a dictionary that can guarantee exact sparse recovery of any $s$-sparse vector by BP and OMP. We discuss how this condition can be computed, as well as its relation with the traditional mutual coherent condition. 

We then discuss in section \ref{sec:sparse-classification} the method of Sparse Representation based Classification (SRC) \cite{Wright:PAMI09}. This method was first proposed for the task of face image classification, in which one is given several aligned face images for each of the several subjects, and the task is to classify any query face image that belongs to one of these subjects. The rationale is that for a Lambertian object, the set of all images taken under varying lighting conditions can be well approximated by a low dimensional subspace. Thus, it is proposed in \cite{Wright:PAMI09} that one uses all the labeled images of all subjects as a dictionary and find a sparse representation of any query image using this dictionary, and the class label is assigned to the group that corresponds to the position of the nonzero entries. The method is generally viewed as an application of spare representation, but it lacks a theoretical justification and there has been discussions and doubts about its effectiveness \cite{Shi:CVPR11,Zhang:ICCV11,Deng:CVPR13,Rigamonti:CVPR11}. In this work, we analyze SRC from the perspective of subspace-sparse recovery, and provide an analysis for it based on our results.

\section{Background}
\label{sec:background}

The purpose of this section is to introduce background for understanding the main results of the paper. We first briefly review the OMP and BP methods for completeness. We then define geometric quantities for charactering the dictionary $\cA$ and talk about their basic properties.

\subsection{Algorithms}

OMP and BP are two methods for sparse recovery. For a dictionary $\cA$ and a signal $\b$, consider the problem

\[
\arg\min_{\x} \|\x\|_0 \st \A \x = \b
\]

OMP is a greedy method that sequentially chooses one dictionary atom in a locally optimal manner. It keeps track of a residual $\v_k$ at step $k$, initialized as the input signal $\b$, and a set $\cW_k$ that contains the atoms already chosen, initialized as the empty set. At each step, $\cW_k$ is updated to $\cW_{k+1}$ by adding the dictionary atom that has the maximum absolute inner product with $\v_k$. Then, $\v_k$ is updated to $\v_{k+1}$ by setting it to be the component of $\b$ that is orthogonal to the space spanned by atoms indexed by $\cW_{k+1}$. The process is terminated when a precise representation of $\b$ is established, i.e., when $\v_k = 0$ for some $k$. 

BP is a convex relaxation approach. The idea is to use the $\ell_1$ norm in lieu of the $\ell_0$ norm, i.e., solve for
\begin{equation}
	P(\cA, \b) := \arg\min_{\x} \|\x\|_1 \st \A \x = \b.
	\label{eq:def-P}
\end{equation}
It has the benefit that \eqref{eq:def-P} is convex and can be solved more efficiently. We will denote the objective value of $P(\cA, \b)$ by $p(\cA, \b)$, and by convention, $p(\cA, \b) = +\infty$ if the problem is infeasible. The dual of the above optimization program is
\begin{equation}
	D(\cA, \b) := \arg\max_\omega \langle \omega, \b \rangle \st \|\A ^\transpose \omega\|_{\infty} \le 1.
	\label{eq:def-D}
\end{equation}
Let $d(\cA, \b)$ be the objective value of the dual problem $D(\cA, \b)$. If the primal problem is feasible, then strong duality holds, i.e., $p(\cA, \b) = d(\cA, \b)$.

\subsection{Sphere and spherical distance}

The spherical distance is defined as the angle between two points in a space $\Re^p \backslash \{\0\}$.
\begin{definition}[Spherical distance]
	The spherical distance $s(\v, \w)$ of two points $\v, \w \in \Re^p \setminus \{\0\}$ is defined as
	\[s(\v, \w):= \cos^{-1} \langle \frac{\v}{\|\v\|_2}, \frac{\w}{\|\w\|_2} \rangle.\]
\end{definition}

The spherical distance is in the range of $[0, \pi]$. For notational convenience, we allow one or both operands of $s(\cdot, \cdot)$ to be sets, in which case the spherical distance is taken to be the infimum of all pairs of points, i.e., for any $\cV \subseteq \Re ^p, \cW \subseteq \Re ^p$,
\[
	s(\cV, \cW) := \inf_{\v \in \cV \setminus \{\0\}}\inf_{\w \in \cW \setminus \{\0\}} s(\v, \w).
\]

Let $\Sp ^{p-1} := \{\v\in \Re^p: \|\v\|_2 = 1\}$ be the set of unit vectors in $\Re ^p$. It is known that $s(\cdot, \cdot)$ defines a metric on $\Sp ^{p-1}$ \cite{Burago:2001}. 

\subsection{Geometric characterization of the dictionary}
\label{sec:def}

The deterministic subspace-sparse recovery conditions rely on geometric properties of the dictionary $\cA$ that characterize the distribution of the atoms in $\cA_0$ and the separation between atoms in $\cA_0$ and $\cA_c$. We first introduce the concept of covering radius. 

\begin{definition}[Covering radius]
	Given the space $\Sp^{p-1}$ with metric $s(\cdot, \cdot)$, the (relative\footnotemark[1]) covering radius of a set of points $\cV \subseteq \Sp ^{p-1}$ is defined as
	\[
	\gamma(\cV):= \max \{ s(\cV, \w): \w\in \spann(\cV) \cap \Sp^{p-1} \}.
	\]
\end{definition}

Intuitively, given a set of points $\cV$, we find a point on the unit sphere of $\spann(\cV)$ that is furthest away from all the points in $\cV$. The name of covering radius also suggests another interpretation, that is, it is the smallest radius such that closed balls of that radius centered at the points of $\cV$ covers all points in $\Sp ^{p-1} \cap \spann(\cV)$. Thus, this concept characterizes how well the points in $\cV$ are distributed, without leaving a large patch of empty region unfilled by any point.

Using this concept, the distribution of the atoms $\cA_0$ is characterized by the covering radius of the set of symmetrized points $\pm\cA_0 := \{\pm\a_i, i \in \cJ_0\}$. We will use the simplified notation $\gamma_0 := \gamma(\pm\cA_0)$. Intuitively, if $\gamma_0$ is small, then there are enough sample points in subspace $\cS_0$, and it should be expected that subspace-sparse recovery should be easier.

Denote $\cK_0 := \conv(\pm \cA_0)$, where $\conv(\cdot)$ is the convex hull of a set of points. It can be identified as a symmetric convex body defined below. 

\begin{definition}[Symmetric convex body]
	A convex set $\cP$ that satisfies $\cP = -\cP$ is called symmetric.
	A compact convex set with nonempty interior is called a convex body.
\end{definition}

\begin{definition}[Polar Set]
	The (relative\footnotemark[1]) polar of a set $\cP$ is defined as $\cP^o = \{\v \in \spann(\cP): \langle\v, \w\rangle \le 1$, $\forall \w \in \cP\}$.
	\label{def:polar-set}
\end{definition}
\footnotetext[1]{It is more convenience to work with the relative quantities in covering radius and polar set since the data $\cA_0$ are in a subspace.}

By this definition, the polar set of $\cK_0$ is given by $\cK_0^o := \{\v \in \cS_0: |\langle \v, \a_i \rangle| \le 1, \forall i \in \cJ_0\}$. Specifically, $\cK_0 ^o$ is also a symmetric convex body, as the polar of a convex body is also a convex body \cite{Brazitikos:14}. 

A subset of the points in $\cK_0^o$ will play a critical role. 

\begin{definition}[Extreme Point]
	A point $\v$ in a convex set $\cP$ is an extreme point if it cannot be expressed as a strict convex combination of two other points in $\cP$, i.e., there are no $\lambda \in (0, 1)$, $\v_1, \v_2 \in \cP$, $\v_1 \neq \v_2$, such that $\v = (1-\lambda)\v_1 + \lambda \v_2$. 
\end{definition}

\begin{definition}[Dual Point]
	The set of dual points of the set $\cA_0$, denoted by $\cD_0$, is defined as the set of extreme points of the set $\cK_0^o$.
	\label{def:dual-point}
\end{definition}

A geometric illustration of some of the definitions is provided in Figure \ref{fig:geometry-2D}. In the following, we discuss some relevant properties for understanding of the concepts and for later use.


\begin{figure*}[t]
	\centering
	\def\nPoint{5}
	\def\primalAngle
	{{21.60, 57.60, 86.40, 126.00, 158.40, 201.60}}
	\def\dualRadius
	{{1.05, 1.03, 1.06, 1.04, 1.08, 1.05}}
	\def\dualAngle
	{{39.60, 72.00, 106.20, 142.20, 180.00, 219.60}}
	\subfigure[Geometry in 2D]
	{
		\begin{tikzpicture}[scale = 2.5]
		\coordinate (0) at (0,0);
		\def\wAngle{0};
		\def\wRefAngle{\primalAngle[0]};
		
		\def\Sx{1.2}         
		\def\Sy{1.2}
		\filldraw[draw=none,fill=gray!20, opacity=0.2] 
		(-\Sx,-\Sy) -- (-\Sx, \Sy) -- (\Sx, \Sy) -- (\Sx, -\Sy) -- cycle;
		\node[black] at (1.1, 1.1, 0) {$\cS_0$};
		\draw[black, fill = none] (0) circle [radius = 1]; 
		\node[left, black] at (0) {O}; 
		\draw [black, fill=black] (0) circle [radius=0.02];
		\foreach \i in {1, 2, ..., \nPoint}
		{
			\def\angle{\primalAngle[\i-1]}
			\draw [pPntColor, fill=pPntColor] (cos \angle, sin \angle) circle [radius=0.02];
			\draw [pPntColor, fill=pPntColor] (-cos \angle, -sin \angle) circle [radius=0.02];
			\draw [dotted, pPntColor] (cos \angle, sin \angle) -- (-cos \angle, -sin \angle);
			\node[above, pPntColor] at (cos \angle, sin \angle) {$\a_\i$}; 
			\node[above, pPntColor] at (-cos \angle, -sin \angle) {$-\a_\i$}; 
		}
		\foreach \i in {1, 2, ..., \nPoint}
		{
			\def\angle{\dualAngle[\i-1]}
			\def\radius{\dualRadius[\i-1]}
			\draw [dPntColor, fill=dPntColor] (\radius * cos \angle, \radius * sin \angle) circle [radius=0.03];
			\draw [dPntColor, fill=dPntColor] (-1 * \radius * cos \angle, -1 * \radius * sin \angle) circle [radius=0.03];
		}
		\foreach \i in {1, 2, ..., \nPoint}
		{
			\draw[solid, pPntColor] (cos \primalAngle[\i-1], sin \primalAngle[\i-1]) -- (cos \primalAngle[\i], sin \primalAngle[\i]); 
			\draw[solid, pPntColor] (-cos \primalAngle[\i-1], -sin \primalAngle[\i-1]) -- (-cos \primalAngle[\i], -sin \primalAngle[\i]); 
		}
		\node[right, pPntColor] at (-0.9, 0.1) {$\cK_0$};
		\foreach \i in {1, 2, ..., \nPoint}
		{
			\draw[solid, dPntColor] (\dualRadius[\i-1] * cos \dualAngle[\i-1], \dualRadius[\i-1] * sin \dualAngle[\i-1]) -- (\dualRadius[\i] * cos \dualAngle[\i], \dualRadius[\i] * sin \dualAngle[\i]); 
			\draw[solid, dPntColor] (-1 * \dualRadius[\i-1] * cos \dualAngle[\i-1], -1 * \dualRadius[\i-1] * sin \dualAngle[\i-1]) -- (-1 * \dualRadius[\i] * cos \dualAngle[\i], -1 * \dualRadius[\i] * sin \dualAngle[\i]); 
		}
		\node[above, dPntColor] at (-1.1, 0.1) {$\cK_0^o$};
		\node[above left, cColor] at (cos \wAngle, sin \wAngle) {$\w$};
		\draw [cColor, fill=cColor] (cos \wAngle, sin \wAngle) circle [radius=0.02];
		\draw[dotted, cColor] (0) -- (cos \wAngle, sin \wAngle);
		\draw[solid, cColor] (0) ([shift=(\wAngle:0.2cm)] 0, 0) arc (\wAngle:\wRefAngle:0.2cm);
		\node[right, cColor] at (0.18, 0.07) {$\gamma_0$};
		\end{tikzpicture}
		\label{fig:geometry-2D}
	}
	~
	\subfigure[Geometry in 3D]
	{
		\tdplotsetmaincoords{70}{10}
		\begin{tikzpicture}[tdplot_main_coords, scale = 3]
		\def\R{1}
		\def\Gamma{21.6}
		
		\draw [black, fill=black] plot [mark=*, mark size=0.6] coordinates{(0, 0, 0)};
		\draw plot [domain = 0:360, samples = 90, variable = \i]
		(\R*cos \i, \R*sin \i, 0) -- cycle;
		\foreach \i in {0, 30,...,150}
		\draw [dotted] plot [domain = 0:360, samples = 60, variable = \j]
		(\R*cos \i*sin \j,\R*sin \i*sin \j, \R*cos \j);
		\foreach \j in {30, 60, ..., 150}
		\draw [dotted] plot [domain = 0:360, samples = 60, variable = \i]
		(\R*cos \i*sin \j,\R*sin \i*sin \j, \R*cos \j);
		\def\Sx{1.3}         
		\def\Sy{1.7}
		\filldraw[draw=none,fill=gray!20, opacity=0.2] 
		(-\Sx,-\Sy,0) -- (-\Sx, \Sy, 0) -- (\Sx, \Sy, 0) -- (\Sx, -\Sy, 0) -- cycle;
		\node[black] at (1.1, 1.1, 0) {$\cS_0$};
		
		\foreach \j in {1, 2, ..., \nPoint}
		{
			\def\angle{\dualAngle[\j-1]}
			\draw [dRgnColor, fill = dRgnColor, fill opacity = 0.5]
			plot [domain = 0:360, samples = 40, variable = \i] 
			(cos \angle * cos \Gamma - sin \angle * sin \Gamma * cos \i,
			sin \angle * cos \Gamma + cos \angle * sin \Gamma * cos \i,
			sin \Gamma * sin \i) -- cycle;
			\draw [dRgnColor, fill = dRgnColor, fill opacity = 0.5]
			plot [domain = 0:360, samples = 40, variable = \i] 
			(- cos \angle * cos \Gamma + sin \angle * sin \Gamma * cos \i,
			- sin \angle * cos \Gamma - cos \angle * sin \Gamma * cos \i,
			sin \Gamma * sin \i) -- cycle;
		}
		\foreach \j in {68.4, 111.6}
		{
			\draw [pRgnColor,very thick] plot [domain = 0:360, samples = 60, variable = \i]
			(\R*cos \i*sin \j,\R*sin \i*sin \j, \R*cos \j);
		}
		\foreach \j in {72, 75.6, 79.2, 82.8, 86.4, 93.6, 97.2, 100.8, 104.4, 108}
		{
			\draw [pRgnColor,dashed,very thick] plot [domain = 0:360, samples = 60, variable = \i]
			(\R*cos \i*sin \j,\R*sin \i*sin \j, \R*cos \j);
		}
		\foreach \j in {1, 2, ..., \nPoint}
		{
			\def\angle{\primalAngle[\j-1]}
			\draw [pPntColor, fill=pPntColor] plot [mark=*, mark size=0.6] coordinates{(cos \angle, sin \angle, 0)};
			\draw [pPntColor, fill=pPntColor] plot [mark=*, mark size=0.6] coordinates{(-cos \angle, -sin \angle, 0)};
		}
		\foreach \j in {1, 2, ..., \nPoint}
		{
			\def\angle{\dualAngle[\j-1]}
			\def\radius{\dualRadius[\j-1]}
			\draw [dPntColor, fill=dPntColor] plot [mark=*, mark size=0.6] coordinates{(\radius * cos \angle, \radius * sin \angle, 0)};
			\draw [dPntColor, fill=dPntColor] plot [mark=*, mark size=0.6] coordinates{(-1*\radius * cos \angle, -1*\radius * sin \angle, 0)};
		}
		\foreach \j in {1, 2, ..., \nPoint}
		{
			\draw[solid, pPntColor] (cos \primalAngle[\j-1], sin \primalAngle[\j-1], 0) -- (cos \primalAngle[\j], sin \primalAngle[\j], 0); 
			\draw[solid, pPntColor] (-cos \primalAngle[\j-1], -sin \primalAngle[\j-1], 0) -- (-cos \primalAngle[\j], -sin \primalAngle[\j], 0); 
		}
		\foreach \j in {1, 2, ..., \nPoint}
		{
			\draw[solid, dPntColor] (\dualRadius[\j-1] * cos \dualAngle[\j-1], \dualRadius[\j-1] * sin \dualAngle[\j-1], 0) -- (\dualRadius[\j] * cos \dualAngle[\j], \dualRadius[\j] * sin \dualAngle[\j], 0); 
			\draw[solid, dPntColor] (-1 * \dualRadius[\j-1] * cos \dualAngle[\j-1], -1 * \dualRadius[\j-1] * sin \dualAngle[\j-1], 0) -- (-1 * \dualRadius[\j] * cos \dualAngle[\j], -1 * \dualRadius[\j] * sin \dualAngle[\j], 0); 
		}
		\draw[dashed, cColor] (0,0,0) -- (cos \dualAngle[3], sin \dualAngle[3], 0);
		\draw[dashed, cColor] (0,0,0) -- (cos \Gamma * cos \dualAngle[3], cos \Gamma * sin \dualAngle[3], sin \Gamma);
		\draw [cColor, fill=cColor] plot [mark=*, mark size=0.6] coordinates{(cos \dualAngle[3], sin \dualAngle[3], 0)};
		\draw [cColor, fill=cColor] plot [mark=*, mark size=0.6] coordinates{(cos \Gamma * cos \dualAngle[3], cos \Gamma * sin \dualAngle[3], sin \Gamma)};
		\draw [cColor] plot [domain = 0:\Gamma, samples = 2, variable = \j]
		(0.3* cos \dualAngle[3] * cos \j, 0.3 * sin \dualAngle[3] * cos \j, 0.3* sin \j);
		\node[left, cColor] at (0.3* cos \dualAngle[3] * cos \Gamma, 0.3 * sin \dualAngle[3] * cos \Gamma, 0.3* sin \Gamma) {$\gamma_0$};
		\end{tikzpicture}
		\label{fig:geometry-3D}
	}
	\caption{Illustration of the geometry of subspace-sparse recovery. Dictionary atoms are $\cA_0:=\{\a_j\}_{j=1}^5$ (drawn in blue) that lie on the unit circle (drawn in black) of a two-dimensional subspace $\cS_0$. Left: illustration of definitions for characterizing $\cA_0$, where the red dots are the dual points. Right: illustration of the geometry of PRC and the DRC, see text for details.  }
\end{figure*}
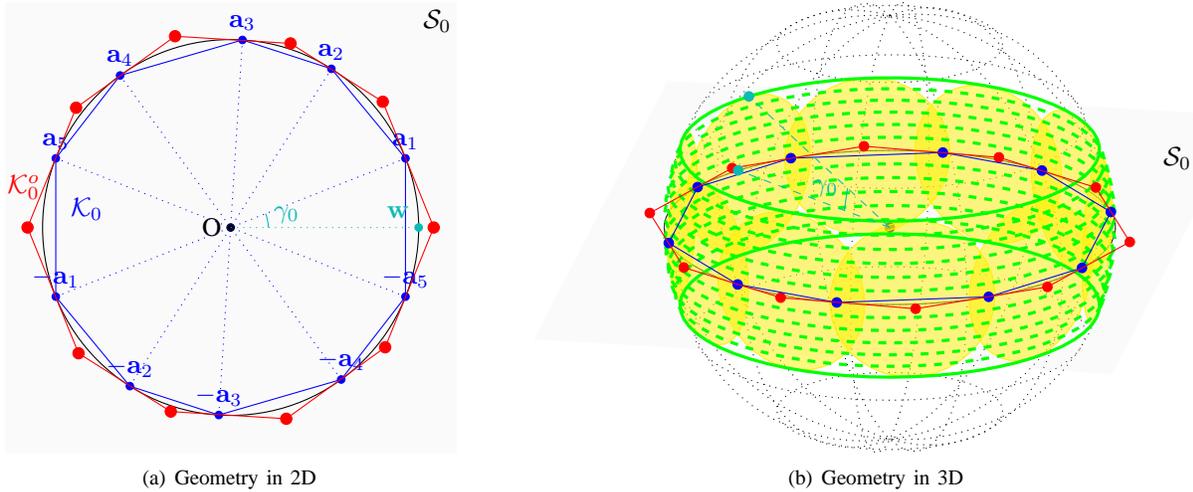

The following result shows that the set $\cK_0^o$ is bounded in terms of the covering radius $\gamma_0$. The intuitive justification is that if $\gamma_0$ is small, then the points $\cA_0$ are dense on the unit sphere, so the polar set $\cK_0^o$ should be smaller. 
\begin{lemma}
	Given $\cA_0$, assume that $\|\a_i\|_2 = 1, \forall i \in \cJ_0$. 
	It has $\max \{ \|\v\|_2: \v \in \cK_0^o \} = 1/\cos \gamma_0$.
	\label{thm:bound-polar}
\end{lemma}

The following result shows that the dual set $\cD_0$ is finite. Essentially, the dual set is composed of the vertices of the polar set $\cK_0^o$.

\begin{lemma}
	Given any $\cA_0$, the set $\cD_0$ is finite. Specifically, 
	\begin{equation}
		\card (\cD_0) \le 2^{d_0} \cdot \binom{s_0}{d_0},
	\end{equation}
	in which $s_0 = \card (\cA_0)$, $d_0 = \dim (\cS_0)$.
	\label{thm:dual-finite}
\end{lemma}

Moreover, all points in $\cK_0^o$ are convex combinations of these finitely many dual points in $\cD_0$. This is implied by the following stronger result.

\begin{lemma}[\cite{Brazitikos:14}]
	The set of the extreme points of a convex body $\cP$ is the smallest subset of $\cP$ with convex hull $\cP$.
	\label{thm:dual-convexhull}
\end{lemma}

\section{Subspace-Sparse Recovery: \\Deterministic Result}
\label{sec:subspace-sparse-deterministic}

In this section, we discuss the theories of subspace-sparse recovery. We start by formally introducing and highlighting the two conditions, PRC and DRC, for guaranteeing the correctness of both OMP and BP for subspace-sparse recovery, then go into details the study of BP and OMP separately.

\begin{figure*}[t]
	\centering
	\begin{tikzpicture}
	\node[box] (1) at(7, 1) {$\forall \b \in \cS_0, \OMP(\cA, \b)$ is subspace-sparse};
	\node[box] (2) at(7, 2.25) {Equivalent condition: $\forall \b \in \cS_0 \setminus \{0\}, s(\cA_0, \{\pm\b\}) < s(\cA_c, \{\pm\b\})$};
	\node[box] (3) at(1, 3.5) {PRC: $\gamma_0 < s(\cA_c, \cS_0)$};
	\node[box] (4) at(7, 3.5) {DRC: $\gamma_0 < s(\cA_c, \cD_0)$};
	\node[box] (5) at(13, 3.5) {$\|\A_c ^\transpose \v\|_\infty < 1, \forall \v \in \cD_0$};
	\node[box] (6) at(7, 4.75) {Equivalent condition: $\forall \b \in \cS_0 \setminus \{0\}, p(\cA_0, \b) < p(\cA_c, \b)$};
	\node[box] (7) at(7, 6) {$\forall \b \in \cS_0, \BP(\cA, \b)$ is subspace-sparse};
	\draw[implies-implies, double equal sign distance] (1) -- (2);
	\draw[       -implies, double equal sign distance] (3) -- (2);
	\draw[       -implies, double equal sign distance] (3) -- (4);
	\draw[       -implies, double equal sign distance] (4) -- (5);
	\draw[       -implies, double equal sign distance] (5) -- (2);
	\draw[       -implies, double equal sign distance] (4) -- (2);
	\draw[       -implies, double equal sign distance] (4) -- (6);
	\draw[       -implies, double equal sign distance] (3) -- (6);
	\draw[       -implies, double equal sign distance] (5) -- (6);
	\draw[implies-implies, double equal sign distance] (6) -- (7);
	\end{tikzpicture}
	\caption{Summary of the results of subspace-sparse recovery with dictionary $\cA = \cA_0 \cup \cA_c$. Each box contains a proposition, and arrows denote implications. The topmost (resp., bottommost) box is the property of subspace-sparse recovery by BP (resp., OMP). Two major conditions for subspace-sparse recovery are the PRC and the DRC. }
	\label{fig:result-flowchart}
\end{figure*}
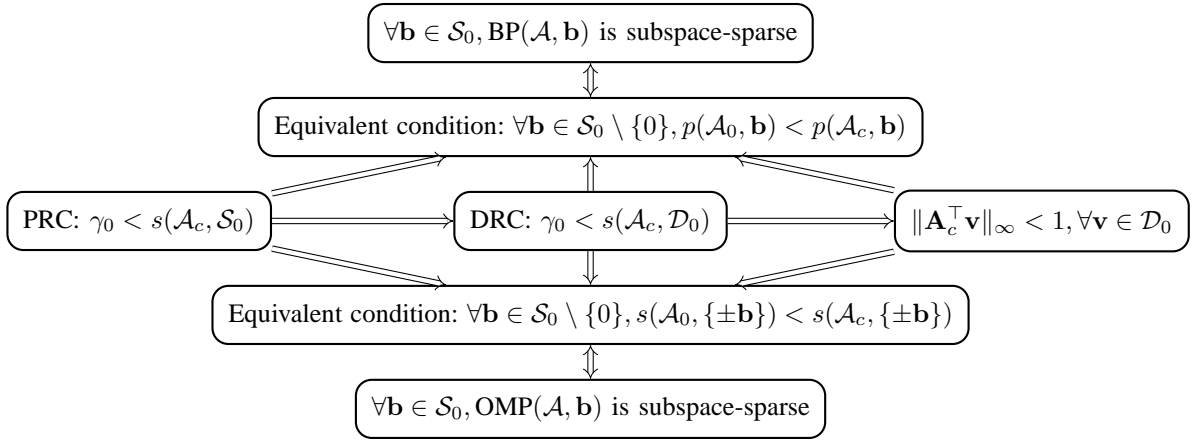

\subsection{Subspace-sparse recovery conditions}

Let $\BP(\cA, \b)$ and $\OMP(\cA, \b)$ be the (sets of) solutions given by the two algorithms. We present conditions under which the solutions $\BP(\cA, \b)$ and $\OMP(\cA, \b)$ are subspace-sparse for all the $\b$ in the subspace $\cS_0$. Concretely, we identify the following two conditions for our objective.

\begin{definition}
	A dictionary $\cA = \cA_0 \cup \cA_c$  is said to satisfy the principal subspace-sparse recovery condition (PRC) if
	\begin{equation}
		\gamma_0 < s(\cA_c, \cS_0),
		\label{eq:PRC}
	\end{equation} 
	in which $\gamma_0$ is the covering radius of $\pm \cA_0$ and $\cS_0$ is the span of $\cA_0$.
	It is said to satisfy the dual subspace-sparse recovery condition (DRC) if
	\begin{equation}
		\gamma_0 < s(\cA_c, \cD_0),
		\label{eq:DRC}
	\end{equation} 
	in which $\cD_0$ is the set of dual points of $\cA_0$.
\end{definition}

The results for subspace-sparse recovery are as follows.

\begin{theorem}
	If PRC is satisfied by a dictionary $\cA = \cA_0 \cup \cA_c$, then $\BP(\cA, \b)$ and $\OMP(\cA, \b)$ are both subspace-sparse for all $\b \in \cS_0$.
	\label{thm:PRC}
\end{theorem}

\begin{theorem}
	If DRC is satisfied by a dictionary $\cA = \cA_0 \cup \cA_c$, then $\BP(\cA, \b)$ and $\OMP(\cA, \b)$ are both subspace-sparse for all $\b \in \cS_0$.
	\label{thm:DRC} 
\end{theorem}

As both theorems show, two major factors affect subspace-sparse recovery. The first is to have the atoms indexed by $\cJ_0$ to be well spread-out across the subspace $\cS_0$, as measured by the covering radius on the left hand side of \eqref{eq:PRC} and \eqref{eq:DRC}. The second factor is that the atoms in $\cA_{c}$ should not be too close to points in $\cS_0$ in the case of PRC or points in $\cD_0$ in the case of DRC. Furthermore, note that PRC requires atoms in $\cA_c$ to be away from \emph{all} points in the subspace $\cS_0$. The DRC, however, is a weaker requirement since it only needs atoms in $\cA_c$ to be away from $\cD_0$, a finite subset of $\cS_0$. Thus, Theorem \ref{thm:PRC} is implied by Theorem \ref{thm:DRC}.

Both PRC and DRC have clear geometric interpretations. Figure \ref{fig:geometry-3D} gives an illustration, in which we show the case of a two dimensional subspace $\cS_0$ in $\Re ^3$. Note that by our assumption, all the atoms of $\cA$ are on the unit sphere shown in the figure. The dictionary $\cA_0$ and the dual points $\cD_0$ are illustrated in blue and red, respectively, see also Figure \ref{fig:geometry-2D} for an illustration in the 2D plane of the subspace $\cS_0$. The two solid green circles have latitude $\pm \gamma_0$ on the unit sphere, they illustrate PRC: the PRC holds if and only if the atoms $\cA_c$ are such that they do not lie in the region enclosed by these two circles (i.e., they all have latitude larger than $\gamma_0$ or smaller than $-\gamma_0$). The DRC is illustrated by the yellow region which is composed of a union of the yellow circles in the space $\Sp ^2$. Each circle is centered at a normalized dual point (note the red dots illustrate the unnormalized dual points) with radius $\gamma_0$. It can be seen that the DRC holds if and only if no point from $\cA_c$ lies in the yellow region. This interpretation generalizes to any subspace dimension $d_0$ and ambient dimension $D$, in which case the PRC and DRC essentially give regions on the unit sphere $\Sp ^{D-1}$ for which the atoms in $\cA_c$ should not reside in. In section \ref{sec:subspace-sparse-randomized} we will revisit this geometric interpretation and analyze under a randomized model the parameters that affect the area of these regions.

The two deterministic results in Theorem \ref{thm:PRC} and \ref{thm:DRC}, alongside with some auxiliary results, are summarized in Figure \ref{fig:result-flowchart}. Each box contains a proposition, and the arrows denote implication relations. The topmost and the bottommost boxes are the properties of subspace-sparse recovery by BP and OMP that we are pursuing. Both of them are implied by the PRC and the DRC. In the following, we give proofs for Theorem \ref{thm:PRC} and \ref{thm:DRC} while at the same time discuss in more detail theories of subspace-sparse recovery by BP and OMP, respectively.

\subsection{Subspace-sparse recovery by BP}

We first establish an equivalent condition for subspace-sparse recovery from BP, then show that this condition is implied by PRC and DRC. See the upper half of Figure \ref{fig:result-flowchart} for an illustration.

\subsubsection{An equivalent condition} There is an equivalent condition for BP to give subspace-sparse solutions. The result appears in the context of subspace clustering \cite{Elhamifar:TPAMI13} and we rephrase the result here for our problem and omit the proof.
\begin{theorem}
	\cite{Elhamifar:TPAMI13}
	$\BP(\cA, \b)$ is subspace-sparse for all $\b \in \cS_0$ if and only if
	$p(\cA_0, \b) < p(\cA_c, \b)$ for all $\b \in \cS_0 \setminus \{0\}$.
	\label{thm:equivalent-BP}
\end{theorem}

In the equivalent condition, it is required that for any $\b \in \cS_0 \setminus \{0\}$, $p(\cA_0, \b)$, which is the objective value of BP for recovering $\b$ by dictionary $\cA_0$ (see \eqref{eq:def-P}), should be smaller than $p(\cA_c, \b)$, which is the objective value of recovering by dictionary $\cA_c$.

\subsubsection{The PRC result} We proceed to discuss how PRC guarantees subspace-sparse recovery by BP. As noted, the PRC implies the DRC, so the PRC result is trivially proved once we show proof for DRC. In the following, we present a direct proof that PRC implies the equivalent condition established in Theorem \ref{thm:equivalent-BP}, as it bears a clearer understanding of PRC for subspace-sparse recovery by BP.

In the equivalent condition, notice that $\b$ is an arbitrary point in $\cS_0$, so the LHS $p(\cA_0, \b)$ depends purely on the properties of $\cA_0$, while RHS $p(\cA_c, \b)$ depends on a relation between the atoms $\cA_c$ and the subspace $\cS_0$. This enlightens us to upper bound the former by characterization of $\cS_0$, and to lower bound the latter by the relation of $\cS_0$ and $\cA_c$.

\begin{theorem}
	If PRC: $\gamma_0 < s(\cA_c, \cS_0)$ holds then $\forall \b \in \cS_c \setminus \{0\}, p(\cA_0, \b) < p(\cA_c, \b)$.
\end{theorem}
\begin{proof}
	We bound the left and right hand sides of the objective inequality separately.
	
	First, notice that $p(\cA_0, \b) = d(\cA_0, \b) = \langle \omega, \b \rangle$ by strong duality, in which $\omega$ is dual optimal solution. Decompose $\omega$ into two orthogonal components $\omega = \omega^\perp + \omega ^\parallel$, in which $\omega ^\parallel \in \cS_0$, it has $\|\A_0 ^\transpose \omega^\parallel\|_2 = \|\A_0 ^\transpose \omega\|_2 \le 1$, where $\A_0$ is a matrix composed of atoms in $\cA_0$ as columns. Thus, by definition of the polar set, $\omega^\parallel \in \cK_0^o$. One can then use Lemma \ref{thm:dual-finite} and get
	\begin{equation}
		p(\cA_0, \b) = \langle \omega ^\parallel, \b \rangle\le \|\b\|_2 \|\omega ^\parallel\|_2 \le \|\b\|_2 / \cos \gamma_0.
		\label{eq:proof-PRC-BP-LHS}
	\end{equation}
	On the other hand, consider the optimization problem 
	\begin{equation}
		P(\cA_c, \b) = \arg\min_{\x} \|\x\|_1 \st \A_c \x = \b,
	\end{equation}
	where $\A_c$ is a matrix composed of atoms in $\cA_c$ as columns. If the problem is infeasible, then the objective of the above optimization $p(\cA_c, \b) = +\infty$, the conclusion follows trivially. Otherwise, take any $\x^* \in P(\cA_c, \b)$ to be the optimal solution, we have $\b = \A_c \x^*$. Left multiply by $\b ^\transpose$ and manipulate the right hand side we have the following:
	\begin{equation}
		\begin{split}
			\|\b\|_2^2 &= \b ^\transpose \A_c \x^* \le \|\A_c ^\transpose \b\|_\infty \|\x^*\|_1\\
			&= \|\A_c ^\transpose \frac{\b}{\|\b\|_2}\|_\infty \|\b\|_2 \cdot p(\cA_c, \b) \\
			&\le \cos s(\cA_c, \cS_0) \cdot \|\b\|_2 \cdot p(\cA_c, \b),
		\end{split}
		\label{eq:proof-PRC-BP-RHS}
	\end{equation}
	so $p(\cA_c, \b) \ge \|\b\|_2 / s(\cA_c, \cS_0)$.
	
	The conclusion thus follows by combining \eqref{eq:proof-PRC-BP-LHS} and \eqref{eq:proof-PRC-BP-RHS} and the condition of PRC.
\end{proof}

\subsubsection{The DRC result} To prove that DRC implies subspace-sparse recovery by BP, we need a statement that is weaker than DRC but is more convenient to work with, see the rightmost box of Figure \ref{fig:result-flowchart}. 

\begin{lemma}
	If DRC: $ \gamma_0 < s(\cA_c, \cD_0)$ holds then it has $\|\A_c ^\transpose \v\|_\infty < 1, \forall \v \in \cD_0$.
\end{lemma}

\begin{proof}
	For any $\v \in \cD_0$, we know that $\v \in \cK_0^o$. Thus, we can use Lemma \ref{thm:dual-finite} to bound $\v$ as $\|\v\|_2 \le 1 / \cos \gamma_0$. Consequently,
	\begin{equation}
			\|\A_c ^ \transpose \v\|_\infty = \|\A_c ^ \transpose \frac{\v}{\|\v\|_2}\|_\infty \|\v\|_2
			 \le \frac{\cos s(\cA_c, \cD_0)}{ \cos \gamma_0} < 1.
	\end{equation}
\end{proof}

\begin{theorem}
	If $\|\A_c ^\transpose \v\|_\infty < 1, \forall \v \in \cD_0$ holds then $\forall \b \in \cS_0 \setminus \{0\}$, $p(\cA_0, \b) < p(\cA_c, \b)$.
	\label{thm:BP-dual}
\end{theorem}

\begin{proof}
	To prove the result, we need some basic results from linear programming. Consider the linear program:
	\begin{equation}
		\arg\max_w \langle \omega, \b \rangle \st \|\A_0 ^\transpose \omega\|_{\infty} \le 1, \omega \in \cS_0.
		\label{eq:Inlier-dual-constraint}
	\end{equation}
	Note that the feasible region of \eqref{eq:Inlier-dual-constraint} is $\cK_0^o$, and it is bounded because it is a convex body. By theories of linear programming (e.g., \cite{Nocedal:06}), there must have a solution to \eqref{eq:Inlier-dual-constraint} that is an extreme point of $\cK_0^o$. Thus, we can always find a solution of \eqref{eq:Inlier-dual-constraint} that is in the set of dual points $\cD_0$.
	
	Now let us consider the optimization problem $D(\cA_0, \b)$, rewritten below for convenience:
	\begin{equation}
		D(\cA_0, \b) := \arg\max_\omega \langle \omega, \b \rangle \st \|\A_0 ^\transpose \omega\|_{\infty} \le 1.
		\label{eq:Inlier-dual}
	\end{equation}
	
	Note that this program differs from \eqref{eq:Inlier-dual-constraint} only in the constraint. The claim is, despite of this change, there is still at least one optimal solution to \eqref{eq:Inlier-dual} that is in $\cD_0$. This follows from the fact that both $\b$ and the columns of $\A_0$ are in $\cS_0$, thus any solution $\omega$ to \eqref{eq:Inlier-dual} can be decomposed into two parts as $\omega = \omega ^\parallel + \omega ^\perp$, in which $\omega ^\parallel$ is a solution to \eqref{eq:Inlier-dual-constraint} and $\omega ^\perp$ is orthogonal to $\cS_0$.
	
	Prepared with the above discussion, we now go to the proof. The proof is trivial if $p(\cA_c, \b) = +\infty$, since $p(\cA_0, \b)$ always has feasible solutions and thus is finite. 
	
	Otherwise, take any $\x^* \in P(\cA_c, \b)$ to be a primal optimal solution. It has that $\b = \A_c \x^*$. On the other hand, we have shown that there exists an optimal dual solution $\omega^* \in D(\cA_0, \b)$ that is in $\cD_0$. Thus,
	\begin{equation}
		\begin{split}
			p(\cA_0, \b) &= d(\cA_0, \b) = \langle \omega^*, \b \rangle =  \langle \omega^*, \A_c \x^* \rangle \\
			&\le \|\A_c ^\transpose \omega^*\|_\infty \cdot \|\x^*\|_1 < p(\cA_c, \b), 
		\end{split}
	\end{equation}
	in which $\|\A_c ^\transpose \omega^*\|_\infty < 1$ by assumption, and $\|\x^*\|_1 = p(\cA_c, \b)$ since $\x^*$ is an optimal solution.
\end{proof}

\subsection{Subspace-sparse recovery by OMP}

The lower half of Figure \ref{fig:result-flowchart} summarizes the results for sparse recovery by OMP. The results surprisingly have a symmetric structure as that of BP.  First, we show an equivalent condition for subspace-sparse recovery by OMP. Then we show that this condition is implied by PRC and DRC.

\subsubsection{An equivalent condition}

\begin{theorem}
	$\forall \b \in \cS_0, \OMP(\cA, \b)$ is subspace-sparse if and only if $\forall \b \in \cS_0 \setminus \{0\}, s(\cA_0, \{\pm\b\}) < s(\cA_c, \{\pm\b\})$.
\end{theorem}
\begin{proof}
	The ``only if'' part is straight forward because if $s(\cA_0, \{\pm\b\}) \ge s(\cA_c, \{\pm\b\})$, then this specific $\b$ will pick a point from $\cA_c$ in the first step of the $\OMP(\cA, \b)$.
	
	The other direction is also easily seen in an inductive way if we consider the procedure of the OMP algorithm. Specifically, for any given $\b \in \cS_0$, the first step of $\OMP(\cA, \b)$ chooses an entry from $\cA_0$, and this gives a residual that is again in $\cS_0$, which then guarantees that the next step of $\OMP(\cA, \b)$ also chooses an entry from $\cA_0$.
\end{proof}

Thus, the equivalent condition requires that for any point $\b \in \cS_0 \setminus \{0\}$, the closest point to either $\b$ or $-\b$ in the entire dictionary $\cA$ should in $\cA_0$.

\subsubsection{The PRC result} Similar to the discussion for BP, the term $s(\cA_0, \{\pm\b\})$ on the LHS of the equivalent condition depends on $\cA_0$ and can be upper bounded by the characterization $\gamma_0$, and the term $s(\cA_c, \{\pm\b\})$ depends on relation between $\cS_0$ and $\cA_c$ and can be bounded below.

\begin{theorem}
	If PRC: $ \gamma_0 < s(\cA_c, \cS_0)$ holds then $\forall \b \in \cS_0 \setminus \{0\}, s(\cA_0, \{\pm\b\}) < s(\cA_c, \{\pm\b\})$.
\end{theorem}
\begin{proof}
	We prove this by bounding each side of the objective inequality separately.
	
	For the left hand side, notice $\gamma_0 := \gamma(\pm \cA_0)$, then by definition of covering radius, $\gamma_0 \ge s(\cA_0, \{\pm\b\})$.
	
	For the right hand side, we have $s(\cA_c, \{\pm\b\}) \ge s(\cA_c, \cS_0)$ by definition of the notation $s(\cdot, \cdot)$.
	
	The conclusion thus follows by concatenating the bounds for both sides above with the PRC. 
\end{proof}

\subsubsection{The DRC result} Finally, we prove the result for DRC, by showing that the statement in the rightmost box of Figure \ref{fig:result-flowchart} guarantees the equivalent condition for OMP.

\begin{theorem}
	If $\|\A_c ^\transpose \v\|_\infty < 1, \forall \v \in \cD_0$ holds then $\forall \b \in \cS_0 \setminus \{0\}, s(\cA_0, \{\pm\b\}) < s(\cA_c, \{\pm\b\})$.
	\label{thm:OMP-dual}
\end{theorem}             

To prove this theorem, we use the result that the polar set $\cK_0^o$ induces a norm on the space $\cS_0$, by means of the so-called Minkowski functional.

\begin{definition}
	The Minkowski functional of a set $\cK$ is defined on $\spann(\cK)$ as
	\begin{equation}
		\|\v\|_\cK = \inf\{t>0: \frac{\v}{t} \in \cK\}.
	\end{equation}
\end{definition}
\begin{lemma}
	\cite{Vershynin:09}
	If $\cK$ is a symmetric convex body, then $\|\cdot\|_\cK$ is a norm on $span(\cK)$ with $\cK$ being the unit ball.
\end{lemma}

By this result, $\|\cdot\|_{\cK_0^o}$ is a norm on $\cS_0$ since $\cK_0^o$ is a symmetric convex body, see the discussion for Definition \ref{def:polar-set}.

\begin{proof}[Proof of Theorem \ref{thm:OMP-dual}]
	It suffices to prove the result for every $\b \in \cS_0 \setminus \{0\}$ that has a unit norm, by using any norm defined on $\cS_0$. Here the norm we use is the Minkowski functional $\|\cdot\|_{\cK_0^o}$, and we need to prove that $s(\cA_0, \{\pm\b\}) < s(\cA_c, \{\pm\b\})$ for all $\b\in \cS_0$ such that $\|\b\|_{\cK_0^o} = 1$.
	
	Since $\|\b\|_{\cK_0^o} = 1$, it has $\b \in \cK_0^o$, by Theorem \ref{thm:dual-convexhull} thus $\b$ could be written as a convex combination of the dual points, i.e. one can write $\b = \sum_i x_i \cdot \v_i $ in which $\v_i \in \cD_0, x_i \in [0, 1]$ for all $i$ and $\sum_i x_i= 1$. Thus, 
	\begin{multline}
	\|\A_c ^\transpose \b\|_\infty = \|\A_c ^ \transpose \sum_i \v_i \cdot x_i \|_\infty
	\le \sum_i \|\A_c ^ \transpose \v_i \cdot x_i \|_\infty\\
	< \sum_i x_i = 1 = \|\A_0 ^\transpose \b\|_\infty,
	\label{eq:proof-OMP-dual}
	\end{multline}
	in which the last equality follows from $\|\b\|_{\cK_0^o} = 1$. One then divide both sides of \eqref{eq:proof-OMP-dual} by $\|\b\|_2$ and take $\arccos$, and the conclusion can be easily seen.
\end{proof}

\section{Subspace-Sparse Recovery: \\Randomized Result}
\label{sec:subspace-sparse-randomized}

In this section, we discuss the properties of subspace-sparse recovery under a randomized model. The analysis is built upon the deterministic condition of DRC in Section \ref{sec:subspace-sparse-deterministic}. We show that under a certain randomized modeling of data, the DRC can be satisfied with certain probabilities. The roadmap of proof of the result is provided.

\subsection{Main result}
\begin{theorem}
	Let $\cA = \{ \a_j \in \Re^D, j \in \cJ \}$ be a dictionary such that $\cA_0$ contains $s_0$ points randomly and uniformly sampled on the unit sphere of some subspace $\cS_0$ with dimension $d_0 < D$, and $\cA_c$ contains points randomly and uniformly sampled on the unit sphere $\Sp^{D-1}$. Let $\rho_0 = s_0 / d_0$ be the ``density'' of points in $\cS_0$, let $\lambda = \card(\cA_c) / s_0$. Under the conditions that $2 \le d_0 < \sqrt{D/2}$ and $\rho_0 \ge 1$, the DRC is satisfied with probability 
	\begin{equation}
		p > 1 -  \frac{d_0\cdot2^{d_0}}{C(D,d_0)} \sqrt{\rho_0} \cdot e^{-C(D, d_0) \sqrt{\rho_0}} - \frac{\lambda d_0 (2e)^{d_0}}{{(\rho_0)} ^{k_0}} ,
		\label{eq:randomized}
	\end{equation} 
	in which $k_0 = \frac{D}{2d_0}-d_0$, $C(D, d_0)$ is increasing in $D$, decreasing in $d_0$ and lower bounded by $0.79 \sqrt{d_0} / 2.07^{d_0-1}$.
	\label{thm:randomized}
\end{theorem}

This theorem asserts that if the dictionary $\cA$ is generated under this random model and satisfies the condition on $d_0$, then both BP and OMP give subspace-sparse recovery for any point $\b \in \cS_0$ with a probability specified in the theorem. The condition that $d_0 \ne 1$ is an artifact introduced by the technique of the proof; one can easily see that if $d_0 = 1$ then subspace-sparse representation can be recovered with probability $1$\footnote{The proof is left as an exercise.}.

Notice that the condition $D > 2{d_0}^2$ requires $D$ to be large and $d_0$ to be small, and as long as the condition is satisfied, the guaranteed probability of success also increases as $D$ increases and as $d_0$ decreases (for large enough $\rho_0$). This conforms with the previous observations that the subspace-sparse recovery works better in cases of low dimensional subspace in high dimensional ambient space \cite{Soltanolkotabi:AS13}. Moreover, the probability is a decreasing function of $\lambda$, showing that subspace-sparse recovery becomes harder if more points are added to $\cA_c$. Finally, the probability goes to $100\%$ as the sample density $\rho_0$ goes to infinity, thus one can achieve arbitrary confidence in getting subspace-sparse recovery by increasing the number of samples to be large enough.

\subsection{Geometric interpretations}

We continue the discussion of the geometric interpretations of DRC in Section \ref{sec:subspace-sparse-deterministic} and analyze the factors that affect the geometry of the problem  under the randomized model in Theorem \ref{thm:randomized}.

We first introduce some definitions. Recall that we use the notation $\Sp^{p-1} = \{\v \in \Re^p: \|\v\|_2 = 1\}$ to denote the unit sphere. Denote $\sigma_{p-1}$ to be a uniform area measure on $\Sp ^{p-1}$. For a given $\w \in \Sp^{p-1}$ and a $\theta \in [0, \pi]$, the spherical cap is a subset of $\Sp^{p-1}$ which is defined as

\begin{equation}
	\Sp _\theta ^{p-1} (\w) = \{\v \in \Sp^{p-1}, s(\w, \v) \le \theta\}.
	\label{eq:def-spherical-cap}
\end{equation}

By this definition, each yellow circle in Figure \ref{fig:geometry-3D} is a spherical cap $\Sp _{\gamma_0} ^{D-1} (\w), \w \in \cD_0$, and the DRC requires that the points in $\cA_c$ do not lie in the union of these spherical caps. With a random sampling of points in $\cA_c$, the chance that DRC is satisfied is determined by 

\begin{equation}
	\frac{\sigma_{D-1} \big( \cup_{\w \in \cD_0} \Sp _{\gamma_0} ^{D-1} (\w) \big)  }{\sigma_{D-1} \big(\Sp ^{D-1}\big)},
	\label{eq:DRC-area-ratio}
\end{equation}
which is the area of the spherical caps relative to the area of $\Sp ^{D-1}$. Obviously, if the quantity in \eqref{eq:DRC-area-ratio} is smaller then the DRC is easier to be satisfied. 

Consider increasing $D$ with all other parameters fixed in the randomized model of Theorem \ref{thm:randomized}. Note that the number and the radius of the spherical caps $\Sp _{\gamma_0} ^{D-1} (\w), \w \in \cD_0$ are all statistically independent of $D$, so we consider $\card(\cD_0)$ and $\gamma_0$ as fixed. It is known that the area of a spherical cap relative to the entire sphere, i.e. $\sigma_{D-1} \big(\Sp _{\gamma_0} ^{D-1}  (\cdot)\big) / \sigma_{D-1} \big(\Sp ^{D-1}\big)$ becomes smaller for higher dimension $D$\footnote{This is known as the phenomenon of concentration of measure, see, e.g. \cite{Ball:97}. This can also be seen from Lemma \ref{thm:covering-radius-bound}, which shows that the radio of area is upper bounded by $\sin^{D-1} \gamma_0 \cdot v_{D-1}/v_D $, which goes to $0$ as $D$ increases to infinity.}. Thus, as $D$ increases, the yellow region given by DRC decreases, and DRC becomes easier to be satisfied.

Consider now that $D$ is fixed and $d_0$ is varied. Intuitively, given a fixed number of points, it is easier to ``cover'' a lower dimensional the unit sphere $\Sp ^{d_0-1}$. In other words, the covering radius $\gamma_0$ decreases as $d_0$ decreases. Thus, decreasing $d_0$ has the effect of shrinking the yellow spherical caps in Figure \ref{fig:geometry-3D}, making DRC easier to be satisfied. 

\subsection{Roadmap of proof}

We provide a roadmap of proof for Theorem \ref{thm:randomized}. This is achieved by providing probabilistic bounds on both sides of DRC separately. In the following, we start by presenting relevant geometric results.

\subsubsection{Preliminary geometric results}

Let $B^p(r) := \{ \v \in \Re^p: \|\v\|_2 \le r \}$ be a ball of radius $1$ in space $\Re ^p$. It is well known that its volume is computed in closed form, i.e., 
\begin{equation}\
\label{eq:v_p}
\vol(B^p(r)) = v_p \cdot r^p, \text{~where~} v_p = \pi ^{\frac{p}{2}} / \Gamma(\frac{p}{2} + 1)
\end{equation}
in which $\vol(\cdot)$ denotes the volume, and $\Gamma(\cdot)$ is the Gamma function.

Based on this, we can further estimate the area of the spherical cap defined in \eqref{eq:def-spherical-cap} by the following result.

\begin{lemma}
	For any $\theta \in [0, \pi/2]$ and any $p \ge 2$,
	\begin{equation}
	\frac{v_{p-1}}{p v_p}
	\sin ^{p-1} \theta
	\le \frac{\sigma_{p-1}(\Sp _\theta ^{p-1} (\w))}{\sigma_{p-1}(\Sp ^{p-1})} 
	\le \frac{v_{p-1}}{v_p}
	\sin ^{p-1} \theta,  
	\label{eq:spherical_cap}
	\end{equation}
	in which $v_p$ is defined in \eqref{eq:v_p}.
	\label{thm:spherical-cap}
\end{lemma}

Equipped with this result, one can give a probabilistic lower bound on the RHS of DRC as follows.

\subsubsection{A lower bound on RHS of DRC}

Notice that according to the probabilistic model in Theorem \ref{thm:randomized}, an arbitrary point $\v \in \cD_0$ and an arbitrary point in $\w \in \cA_c$ are independent. Moreover, the point $\w$ is uniformly distributed on the unit sphere, so the effect of the angle $s(\w, \v)$ is as if holding $\v$ fixed and letting $\w$ as uniformly distributed on $\Sp^{D-1}$ at random. By using upper bound on the area of spherical cap in \eqref{eq:spherical_cap}, one can get for any $\gamma^* \in [0, \pi/2]$ that $P(s(\v, \w) > \gamma^*) \ge 1 - \frac{v_{D-1}}{v_D} \sin^{D-1}\gamma^*$. One can then apply union bound on all pairs of points $\cD_0 \times \cA_c$. Notice $\card(\cD_0) \le \binom{s_0}{d_0} \cdot 2^{d_0}$ by Lemma \ref{thm:dual-finite} and $\card(\cA_c) = \lambda \cdot s_0$, we get

\begin{equation}
P(s(\cD_0, \cA_c) > \gamma^*) \ge 1 - \lambda s_0 \cdot \binom{s_0}{d_0} 2^{d_0} \cdot \frac{v_{D-1}}{v_D} \sin ^{D-1}\gamma^*.
\label{eq:bound-on-right}
\end{equation}

We are left to give an upper bound on the LHS of DRC. Essentially, we need to give a probabilistic bound on the covering radius.

\subsubsection{An upper bound on covering radius}

Given the unit sphere $\Sp ^{p-1}$ and a positive integer $M$, we consider the problem that if there are $M$ points independently and uniformly drawn from the sphere $\Sp ^{p-1}$ at random, how well-spread out they are in terms of covering radius. Intuitively, as more points are sampled, the unit sphere is expected to be better covered by the samples and the covering radius is expected to be smaller. In the following, we give a rigorous statement of this intuition and proofs are delayed to appendix. Our proof draws inspiration from the work \cite{Heckel:arxiv13}. The idea is simple: assume that there is a set of circles of radius $\epsilon$ on $\Sp ^{p-1}$ that can cover the entire unit sphere (i.e., an $\epsilon$-covering as defined below), if the $M$ sample points are distributed on $\Sp ^{p-1}$ in a way that every small circle contains at least one sample point, then the covering radius can be bounded by $2 \times \epsilon$. Before discussing how this is realized, we first introduce two definitions.

\begin{definition}
	A set $\cV \subseteq \Sp^{p-1}$ is called an $\epsilon$-covering of $\Sp^{p-1}$ if the covering radius of $\cV$ is no more than $\epsilon$. Given $\epsilon > 0$, the covering number of $\Sp^{p-1}$, denoted by $\mathcal{C}(\Sp^{p-1}, \epsilon)$ is the cardinality of the smallest $\epsilon$-covering of $\Sp^{p-1}$.
\end{definition}

First, it is desirable to find an $\epsilon$-covering of $\Sp^{p-1}$ with as small cardinality as possible. 
\begin{lemma}
	The covering number of $\Sp^{p-1}, p \ge 2$ is bounded by 
	\[
	\mathcal{C}(\Sp^{p-1}, \epsilon) \le \frac{p}{\frac{v_{p-1}}{v_p}\sin^{p-1}\frac{\epsilon}{2}}, \forall \epsilon \le \frac{\pi}{4}.
	\]
	\label{thm:covering-number}
\end{lemma}

Given this, we further lower bound the probability that every circle in the $\epsilon$-covering contains at least one sample point, and the bound on covering radius can be obtained.
\begin{theorem}
	Let $\cP \subseteq \Sp^{p-1}, p \ge 2$ be a set of $K$ points that are drawn independently and uniformly at random on $\Sp^{p-1}$. Then for any $\gamma^* \le \pi/2$, it has $\gamma(\pm \cP) < \gamma^*$ with probability at least $1 - \frac{p v_p}{v_{p-1}}\frac{1}{\sin ^{p-1} \frac{\gamma^*}{4}} \exp (-K \frac{2 v_{p-1}}{p v_p}\sin ^{p-1} \frac{\gamma^*}{2})$
	\label{thm:covering-radius-bound}
\end{theorem}

With this result, the LHS of DRC is upper bounded by the following.

\begin{multline}
	P(\gamma_0 < \gamma^*) \ge 1 -  \frac{d_0 \cdot v_{d_0}}{v_{(d_0-1)}} \cdot\frac{1}{\sin^{(d_0-1)}\frac{\gamma^*}{4}} \\ \cdot \exp (-s_0 \frac{2 v_{(d_0-1)}}{v_{d_0}}\sin^{(d_0-1)} \frac{\gamma^*}{2}),
	\label{eq:bound-on-left}
\end{multline}
for every $\gamma^* \in (0, \frac{\pi}{2}]$.

\subsubsection{Final proof} One can see that by combining \eqref{eq:bound-on-right} and \eqref{eq:bound-on-left}, we can get a probability that $\gamma_0 < s(\cD_0,\cA_c)$ in terms of the parameter $\gamma^* \in (0, \frac{\pi}{2}]$. The result in \eqref{eq:randomized} is subsequently acquired by taking a specific value of $\gamma^*$. The details are deferred to appendix.

\section{Applications}
\label{sec:applications}

In this section, we apply the theoretical results in the previous sections to the analysis of the traditional sparse recovery. In this process, we also establish the relation between the PRC/DRC and the mutual coherent condition in sparse recovery. Moreover, we also discuss the application of our results to the analysis of sparse representation based recognition.

\subsection{Sparse Recovery}
\label{sec:sparse-recovery}

In sparse recovery, the task is to reconstruct an $s_0$-sparse signal $\x$ (i.e. $\x$ has at most $s_0$ nonzero entries) from the observation $\b = \A \x$ for some dictionary $\cA$. In order to analyze the problem by the subspace-sparse representation results, we take the set $\cJ_0$ to be the $s_0$ columns corresponding to the nonzero entries of $\x$ and get a partition of $\cA$ into $\cA_0 \cup \cA_c$. If $\cA_0$ has the property that its atoms are linearly independent, then $\x$ is the unique subspace-sparse solution. In this case, subspace-sparse recovery and subspace-sparse recovery are equivalent, in the sense that if one guarantees finding subspace-sparse representation, then correct sparse recovery can be achieved. Consequently, by using our PRC and DRC results, we can have the following result.

\begin{theorem}
	Given a dictionary $\cA$, any $s_0$-sparse vector $\x$ can be recovered from the observation $\b = \A \x$ by BP and OMP if for any partition of $\cA$ into $\cA_0$ and $\cA_c$ where $\card(\cA_0) = s_0$, it has that atoms in $\cA_0$ are linearly independent and that PRC (respectively, DRC) holds.
	\label{thm:sparse-recovery}
\end{theorem}

This result serves as a new condition for guaranteeing reconstruction of sparse signals. Its geometric interpretation is the same as that of PRC and DRC for the subspace-sparse recovery, i.e., for any $s_0$ atoms of the dictionary, they should be well distributed in their span, while all other atoms should be sufficiently away from this span (by PRC) or from a subset of the span (by DRC).

For the purpose of checking the conditions of the theorem, if any $s_0$ atoms in $\cA$ are linearly independent, then subsequent checking of the PRC and DRC is easy, as explained below. First, the dual points $\cD_0$ can be written out explicitly:

\begin{lemma}
	For $\cA_0$ which has $s_0$ linearly independent atoms, the set of dual points, $\cD_0$, contains exactly $2^{s_0}$ points specified by $\{\A_0 (\A_0^\transpose \A_0) ^{-1} \cdot \u, \u \in U_{s_0}\}$, where $U_{s_0} := \{ [u_1, \cdots, u_{s_0}], u_i = \pm 1, i = 1, \cdots, s_0 \}$.
	\label{thm:sparse-recovery-dualpoint}
\end{lemma}

The proof is in the appendix. With the dual points, one can then compute $s(\cA_c, \cS_0)$ and $s(\cA_c, \cD_0)$ on the RHS of PRC and DRC. Moreover, the covering radius $\gamma_0$ can also be computed by the relation in Lemma \ref{thm:bound-polar}, i.e.
\begin{multline}
	\cos \gamma_0 =  1 / \max\{ \|\v\|_2: \v \in \cK_0^o \} \\
	 =1 / \max\{ \|\v\|_2: \v \in \cD_0 \},
\end{multline}
where the last equality follows from the fact that $\cD_0$ is the set of extreme points of $\cK_0^o$. Thus, all terms in PRC and DRC can be computed. 

At the end of this section, we point out that the result of Theorem \ref{thm:sparse-recovery} can be compared with traditional sparse recovery results. Specifically, we compare it with the result that uses mutual coherence, $\mu(\cA)$, which is defined as the largest absolute inner product between atoms of $\cA$. It is known that $\mu(\cA) < \frac{1}{2{s_0}-1}$ is a sufficient condition for OMP and BP \cite{DonohoElad:PNAS03,Tropp:TIT04} to recover $s_0$-sparse signals. We show that this is a stronger requirement than that of Theorem \ref{thm:sparse-recovery}.

\begin{theorem}
	If a dictionary $\cA$ satisfies $\mu(\cA) < \frac{1}{2{s_0}-1}$, then for any partition of $\cA$ into $\cA_0$ and $\cA_c$ where $\card(\cA_0) = s_0$, it has that the atoms in $\cA_0$ are linearly independent and that PRC and DRC hold.
	\label{thm:sparse-recovery-compare}
\end{theorem}

The proof is in the appendix. This result shows that the PRC/DRC conditions in Theorem \ref{thm:sparse-recovery} are implied by the condition of mutual coherence. While the mutual coherence condition requires all atoms of $\cA$ to be incoherent from each other, the PRC and DRC provide more detailed requirements, in terms of the distribution of points $\cA_0$ as well as the relation of $\cA_0$ and $\cA_c$.

\subsection{Sparse Classification}
\label{sec:sparse-classification}

We can use the deterministic and randomized results for subspace-sparse recovery for the analysis of the sparse representation based classification (SRC) method. Assume that we are given a dictionary $\cA := \{ \a_j, j \in \cJ\}$ which contains data from a union of $n$ subspaces, i.e., there exist a partition of $\cJ$ into $\cJ_1, \cdots, \cJ_n$, such that any two different set $\cJ_i$ and $\cJ_j$ do not intersect and that $\cup_i \cJ_i = \cJ$, and that $\cA_i := \{ \a_j, j \in \cJ_i \}$ contains points from a low dimensional subspace $\cS_i$. Following the notational tradition, we assume that the $i$-th group has $s_i$ points in subspace of dimension $d_i$, and the geometric quantities of $\gamma_i, \cK_i, \cK_i^o$ and $\cD_i$ can all be defined.

The task in the classification is that given this dictionary $\cA$ where we have an explicit knowledge of the partition $\{ \cJ_i \}_{i=1}^n$, we want to find the membership of any other point $\b$ that lies in the union of subspaces $\cup_{i=1}^n \cS_i$ determined by which specific subspace it belongs to\footnote{We assume that any two subspaces intersect only at the origin, so that such membership is unique.}. In the work of \cite{Wright:PAMI09}, the authors proposed the SRC which finds a sparse representation of $\b$ as in \eqref{eq:L0} by BP or OMP\footnote{While it is proposed to use BP in \cite{Wright:PAMI09}, the idea can be easily extended to using OMP. We study both of them.}. Ideally, the coefficient vector $\x$ for representing $\b$ is subspace-sparse, i.e. is such that the nonzero entries of $\x$ are all in the set $\cJ_i$ in which $i$ is the index of the subspace that $\b$ belongs to, so the query $\b$ can be correctly classified. Other techniques are proposed for SRC to robustify the method so that one can classify a point when the representation $\x$ has nonzero coefficients in two or more groups. however, we analyze here the conditions for guaranteeing subspace-sparse recovery, which is sufficient for SRC to give the correct class label.

First, our result of PRC in Theorem \ref{thm:PRC} and DRC in Theorem \ref{thm:DRC} can be easily applied here for analyzing when a correct classification can be guaranteed. Here, we use the DRC result, and formulate the following theorem.

\begin{theorem}
	Given $\cA = \cup_{i=1}^n \cA_i$, assume $\|\a_j\|_2=1, \forall \a_j \in \cA$, subspace classification by BP and OMP succeeds for any point $\b \in \cup_{i=1}^n \cS_i$ if
	\begin{equation}
	\gamma_i < s(\cD_i, \cA \backslash \cA_i), \forall i = 1, \cdots, n,
	\end{equation}
	in which $\gamma_i$ is the covering radius of $\pm \cA_i$, $\cD_i$ is the set of dual points of $\cA_i$, the backslash in $\cA \backslash \cA_i$ denotes the set different.
	\label{thm:sparse-classification-deterministic}
\end{theorem}

This theorem asserts that we need the dictionary to have well-distributed points in each of the subspaces so that $\gamma_i$ is small. Also, the dual points $\cD_i$ which are in subspace $\cS_i$ need to be not too close to points in all other subspaces.

We can also formulate a randomized result.

\begin{theorem}
	Suppose there are $n$ subspaces $\cS_i$ with dimensions $d_i$ chosen independently and uniformly at random in $\Re ^D$. Suppose that $s_i$ points are sampled independently and uniformly at random on each of the $n$ subspaces. Let $\rho_i := s_i / d_i$ and $p_i := s_i / \sum_j s_j$ be the density of points and proportion of point in subspace $i$, respectively. Then any $\b \in \cup_{i=1}^n \cS_i$ can be correctly classified by BP and OMP if $2 \le d_i < \sqrt{D/2}$ and $\rho_i \ge 1$, $i = 1, \cdots, n$, with probability
	\begin{equation}
	p > 1 - \sum_i \Big( \frac{d_i\cdot 2^{d_i}}{C(D,d_i)} \sqrt{\rho_i} \cdot e^{-C(D, d_i) \sqrt{\rho_i}} + \frac{d_i (2e)^{d_i}}{p_i (\rho_i) ^{k_i}} \Big) ,
	\label{eq:sparse-classification-randomized}
	\end{equation}
	where $k_i = \frac{D}{2 d_i} -d_i$, and $C(D, d_i)$ is a constant as before.
	\label{thm:sparse-classification-randomized}
\end{theorem}

This result shows that classification based on subspace-spares recovery is expected to work if subspace dimension is small and ambient dimension is large, and there should be enough number of samples in each subspace.

\section{Related Works and Future Directions}

\subsection{Related works and comparison}
Prior to this work, there has been studies of subspace sparse recovery by BP \cite{Elhamifar:ICASSP10,Soltanolkotabi:AS13} and by OMP \cite{Dyer:JMLR13} in the context of subspace clustering. In this section, we compare our results with these works by trying to reformulate or applying their results to the analysis of the subspace-sparse recovery problem considered in this work.

Theorem 1 in \cite{Elhamifar:ICASSP10} gives a sufficient condition for the correctness of subspace clustering by BP. While the condition it gives is in terms of a dictionary composed of several subspaces, we can apply it to our problem by taking points from one specific subspace as $\cA_0$, and all points from all other subspaces as $\cA_c$. The result is that subspace-sparse recovery by BP can be achieved for all $\b \in \cS_0$ if the following is true:
\begin{equation}
	\max_{\tilde{\A}_0 \in \cW_0} \sigma_{d_0} (\tilde{\A}_0) / \sqrt{d_0} > \cos s(\cS_0, \spann(\cA_c)),
	\label{eq:prior-Ehsan}
\end{equation}
where $\cW_0$ is the set of all full column rank submatrices $\tilde{\A}_0$ of $\A_0$. The LHS of the condition \eqref{eq:prior-Ehsan} is not well interpretable, and it is later observed by \cite{Soltanolkotabi:AS13} that the LHS can be bounded as $\max_{\tilde{\A}_0 \in \cW_0} \sigma_{d_0} (\tilde{\A}_0) / \sqrt{d_0} \le \cos \gamma_0$\footnote{\cite{Soltanolkotabi:AS13} shows that the LHS $\le r(\cK_0)$, where $r(\cdot)$ is the inradius. To get to the claim, we then use the fact that $r(\cK_0) = \cos \gamma_0$, which is a trivial consequence of Lemma 7.3 in \cite{Soltanolkotabi:AS13} and Lemma \ref{thm:bound-polar} in this paper.}. For the RHS of \eqref{eq:prior-Ehsan}, one can easily get $\cos s(\cS_0, \spann(\cA_c)) \ge \cos s(\cS_0, \cA_c)$. Thus, the condition \eqref{eq:prior-Ehsan} is more restrictive than both PRC and DRC. Actually, the condition \eqref{eq:prior-Ehsan} may be too restrictive in most cases, since the RHS will be equal to $1$ (while the LHS is at most $1$) unless $\spann(\cA_c)$ intersects with the subspace $\cS_0$ only at the origin. 

The deterministic analysis in \cite{Soltanolkotabi:AS13} considers a slightly different problem than that of this paper. Concretely, it considers the subspace-sparse recovery of a \emph{specific} $\b \in \cS_0$ rather than for \emph{all} points in $\cS_0$. It asserts that if\footnote{We have used the fact that $r(\cK_0) = \cos \gamma_0$, see the previous footnote.}
\begin{equation}
	\gamma_0 < s(\cA_c, \{\pm \v\}),
	\label{eq:prior-Mahdi}
\end{equation}
then BP gives subspace-sparse solution for $\b$. In the formula, $\v$ is the so-called ``dual point'' (we will see that this \emph{``dual point''} is related to our definition of the set of the dual point in Definition \ref{def:dual-point}), which is any solution to the program in \eqref{eq:Inlier-dual-constraint}. Notice that $\v$ is in $\cS_0$ by this definition.

To compare this with our result, we apply it to all possible $\b$'s that are in subspace $\cS_0$, and get the condition
\begin{equation}
	\gamma_0 < s(\cA_c, \pm \cV_0),
	\label{eq:prior-Mahdi-reformulated}
\end{equation}
in which $\cV_0 = \{ \v: \text{dual point of }\b, \forall \b \in \cS_0 \}$. Thus, equation \eqref{eq:prior-Mahdi-reformulated} is a condition for subspace-sparse recovery for all $\b \in \cS_0$, and is now comparable to PRC and DRC. However, the structure of $\cV_0$ is unknown; the best one can do is to take it to be $\cS_0$ since the only knowledge about $\v$ is that each of them is in $\cS_0$. By doing this, the condition \eqref{eq:prior-Mahdi-reformulated} becomes the PRC. To further refine this result, one needs to investigate the structure of the set $\cV_0$. It is shown in the proof to Theorem \ref{thm:BP-dual} that for any $\b$, $\v$ can be taken as a point in the set $\cD_0$. Thus, the set of dual points $\cD_0$ as defined in Definition \ref{def:dual-point} is composed of all ``dual points'' $\v$ as defined in the work of \cite{Soltanolkotabi:AS13}. The contribution of our work on the basis of \cite{Soltanolkotabi:AS13} can be viewed as specifying the structure of the set $\cV_0$ in \eqref{eq:prior-Mahdi-reformulated}.

The above two works are for analysis of BP. In \cite{Dyer:JMLR13}, the authors give a deterministic condition for guaranteeing correct subspace-sparse recovery by OMP. Their condition can be formatted to our notation as
\begin{multline}
	\cos s(\cA_0, \cA_c) < \cos \gamma_0 \\- \frac{2}{\sqrt[4]{12}}\sqrt{1-\cos^2\gamma_0} \cos s(\cS_0, \spann(\cA_c)),
	\label{eq:prior-Dyer}
\end{multline}
and if this condition holds, then OMP can achieve subspace-sparse solution for any $\b \in \cS_0$. The LHS of \eqref{eq:prior-Dyer} characterizes the spherical distance between the points in $\cA_0$ and points in $\cA_c$, and satisfies our intuition that this distance should be large for the purpose of subspace sparse recovery. On the RHS, the term $\cos s(\cS_0, \spann(\cA_c))$ is the same as that on the RHS of \eqref{eq:prior-Ehsan}, and we have argued that this term becomes $1$ unless $\cS_0$ and $\spann(\cA_c)$ have trivial intersection, making the RHS large and the condition difficult to be satisfied. Moreover, it is shown recently that \eqref{eq:prior-Dyer} is implied by PRC \cite{Chong:arxiv15-SSCOMP}. Thus, this condition is more restrictive than PRC and DRC.

\subsection{Future directions and existing works}

The analysis of this paper assumes that the atoms of the dictionary are noise-free. A natural follow-up question is the robustness of the result to corruptions on the dictionary $\cA$ and on the signal $\b$. In the context of subspace clustering by BP, this problem has already been investigated. Specifically, in the works of \cite{Wang-Xu:ICML13} and \cite{Soltanolkotabi:AS14} the authors show that with different modifications on BP, subspace-sparse recovery based clustering is still provably correct. Although this is not a direct study of the subspace-sparse recovery of this paper, it shows evidence that the BP or its variant is likely to be robust to noise. More recently, the work of \cite{Elhamifar:ArXiv14} introduces the idea of approximate subspace-sparse solutions, and shows that under certain conditions, the solution is approximately subspace-sparse. This gives another promising direction to extend the analysis of this paper to noisy case. On the other hand, the performance of subspace-sparse recovery by OMP has not been studied to the best of our knowledge. However, there are results in the study of traditional sparse recovery that show the robustness of OMP to noise \cite{Tropp:TIT04,Donoho:TIT06}. This also shows the possibility of extending OMP for subspace-sparse recovery in noisy cases.

\section{Conclusion}
\label{sec:conclusion}
In this work, we have studied the properties of OMP and BP algorithms for the task of subspace-sparse recovery and have identified the PRC and DRC as two sufficient conditions for guaranteeing subspace-sparse recovery. These two conditions reveal that the dictionary atoms within the subspace need to be well-distributed, and atoms outside of the subspace need to be not too close to the subspace (by PRC) or to the set of dual points in the subspace (by DRC). We further show that with a random modeling of the dictionary, the DRC is expected to hold if subspace dimension is low and ambient dimension is high. We have applied our results to the analyses of traditional sparse recovery as well as in sparse representation based classification. Especially, we have shown that our result not only provides guarantees for the correctness the sparse recovery problem, but the condition is relaxed than that given by mutual coherent.

\newpage
\appendices
\section{Proof of lemmas in section \ref{sec:background}}

\subsection{Proof of Lemma \ref{thm:bound-polar}}
\begin{lemma*}
	Assume that $\|\a_i\|_2 = 1, \forall i \in \cJ_0$. 
	It has $\max \{ \|\v\|_2: \v \in \cK_0^o \} = 1/\cos \gamma_0$.
\end{lemma*}
\begin{proof}
	By the definitions of $\cK_0^o$ and $\gamma_0$, the conclusion of the lemma can be written as
	\begin{equation}
	\max_{\|\A_0^\transpose \v\|_\infty \le 1} \|\v\|_2 = 1 / \min_{\|\v\|_2 = 1} \|\A_0^\transpose \v\|_\infty,
	\end{equation}
	which can be easily seen as true.
\end{proof}

\subsection{Proof of Lemma \ref{thm:dual-finite}}
This lemma is a particular case of a well-known result in linear programming.
\begin{lemma*}
	The set $\cD_0$ is finite. Specifically, 
	\begin{equation}
	\card (\cD_0) \le 2^{d_0} \cdot \binom{s_0}{d_0},
	\end{equation}
	in which $s_0 = \card (\cA_0)$.
\end{lemma*}

\begin{proof}
	Consider a linear program with variable $\v$, constraint $\v \in \cK_0^o$, and arbitrary objective. Since the dual points $\cD_0$ are the extreme points of $\cK_0^o$, they are the same as the basic feasible solutions of the linear program \cite{Nocedal:06}. Assume that the index set $\cJ_0$ contains $s_0$ elements. Each basic feasible solution is determined by $d_0$ linearly independent constraints from the $2\cdot s_0$ constraints of $\|\A_0^\transpose \v\|_\infty \le 1$. Obviously, there are at most $2^{s_0} \cdot \binom{s_0}{d_0}$ ways to choose such set of constraints. .
\end{proof}

\section{Proofs for section \ref{sec:subspace-sparse-randomized}}

\subsection{Proof for Lemma \ref{thm:spherical-cap}}
\begin{lemma*}
	For any $\theta \in [0, \pi/2]$ and any $p \ge 2$,
	\begin{equation}
	\frac{v_{p-1}}{p v_p}
	\sin ^{p-1} \theta
	\le \frac{\sigma_{p-1}(\Sp _\theta ^{p-1} (\w))}{\sigma_{p-1}(\Sp ^{p-1})} 
	\le \frac{v_{p-1}}{v_p}
	\sin ^{p-1} \theta,  
	\end{equation}
	in which $v_p$ is defined in \eqref{eq:v_p}.
\end{lemma*}
\begin{proof}
	The idea is similar to that in \cite{Tkocz:12}. We first prove the upper bound. See Figure \ref{fig:spherical-cap} for an illustration, in which we project $\Re^p$ into any two-dimensional space that contains the origin and $\w$. The potion of the area of the spherical cap over the entire $\Sp ^{p-1}$ is the same as the potion of the volume of the red dashed cone intersecting with $B^p(r)$ over the volume of $B^p(r)$. Also note that the part of the red cone in the $B^p(r)$ lie completely in the green dotted cylinder. Thus,
	\begin{multline}
	\frac{\sigma_{p-1}(\Sp _\theta ^{p-1} (\w))}{\sigma_{p-1}(\Sp ^{p-1})} =
	\frac{\vol(\text{Cone} \cap B^p(1))}{\vol(B^p(1))} \le \frac{\vol(\text{Cylinder})}{\vol(B^p(1))} \\
	= \frac{\sin ^{p-1} \theta \cdot v_{p-1} \cdot 1}{1^p \cdot v_p}
	= \sin^{p-1} \theta \frac{v_{p-1}}{v_p},
	\end{multline}
	this proves the upper bound.
	
	For the lower bound, consider again the part of the red cone in the $B^p(r)$, its volume is bounded below by the intersection of the red and the cyan cones. It is known that the volume of a $p$-dimensional cone (i.e. a cone with a $p-1$ dimensional base) is the product of the $p-1$ dimensional area of its base and its height divided by $p$. Thus, one can see that the volume of the intersection of the two cones is $v_{p-1} \sin^{p-1}\theta \cdot 1 / p$. The conclusion thus follows from this discussion.
\end{proof}
		
\begin{figure}[t]
	\centering
	\begin{tikzpicture}[scale = 2]
	\coordinate (0) at (0,0);
	
	\coordinate (x) at (1,0);
	\coordinate (x1) at (1.3,1.3);
	\coordinate (x2) at (1.3,-1.3);
	\coordinate (y1) at (0.46,1.3);
	\coordinate (y2) at (0.46,-1.3);
	
	\draw[solid, black] (0) ([shift=(-120:1cm)] 0, 0) arc (-120:120:1cm);
	
	\node[left, blue] at (0) {O}; 
	\draw [blue, fill=blue] (0) circle [radius=0.01];
	
	\node[right, blue] at (x) {$\w$}; 
	\draw [blue, fill=blue] (x) circle [radius=0.01];
	
	\draw[solid, blue] (0) -- (x); 
	\draw[dashed, red] (0) -- (x1); 
	\draw[dashed, red] (0) -- (x2);
	\draw[dashed, cyan] (x) -- (y1); 
	\draw[dashed, cyan] (x) -- (y2);
	
	\draw[solid, blue] (0) ([shift=(0:0.2cm)] 0, 0) arc (0:45:0.2cm);
	\node[above right, blue] at (0.2, 0) {$\theta$};
	
	\coordinate (bl) at (0, -0.707); 
	\coordinate (al) at (0, 0.707);
	\coordinate (br) at (1, -0.707);
	\coordinate (ar) at (1, 0.707); 
	
	\draw[dotted, green] (bl) -- (al);
	\draw[dotted, green] (bl) -- (br);
	\draw[dotted, green] (al) -- (ar);
	\draw[dotted, green] (br) -- (ar);
	
	\end{tikzpicture}
	\caption{Illustration for proving bounds for area of spherical cap.}
	\label{fig:spherical-cap}
\end{figure}
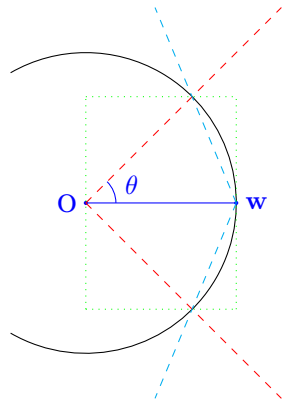

\subsection{Proof for Lemma \ref{thm:covering-number}}

\begin{lemma*}
	The covering number of $\Sp^{p-1}, p \ge 2$ is bounded by 
	\[
	\mathcal{C}(\Sp^{p-1}, \epsilon) \le \frac{p}{\frac{v_{p-1}}{v_p}\sin^{p-1}\frac{\epsilon}{2}}, \forall \epsilon \le \frac{\pi}{4}.
	\]
\end{lemma*}
\begin{proof}
	A standard way of bounding covering number is to construct a specific $\epsilon$-covering $\cV$. Concretely, initialize $\cV$ as empty. In the first step, add an arbitrary point in $\Sp^{p-1}$ into $\cV$. In the following steps, find any point $\w$ in $\Sp^{p-1}$ which satisfy $s(\w, \cV) > \epsilon$ and add this $\w$ into $\cV$. The procedure is terminated when no such point exists.
	
	It is easy to see that this procedure must terminate in finite number of iterations. In fact, we will provide an upper bound on the number of iterations.
	
	Before that, we first point out that the $\cV$ constructed in this way is an $\epsilon$-covering of $\Sp^{p-1}$, or equivalently, $\gamma(\cV) \le \epsilon$. Otherwise, there would be a $\w$ such that $s(\w, \cV) > \epsilon$, and by the procedure above, this $\y$ should be added to $\cV$. Thus, we can bound the covering number $\mathcal{C}(\Sp^{p-1}, \epsilon)$ by the cardinality of $\cV$ that we constructed above.
	
	We now give a bound on $\card(\cV)$. Imagine that centered at each point in $\cV$ we draw a ball (in the space of $(\Sp^{p-1}, s(\cdot, \cdot))$) with radius $\epsilon / 2$. Then by the construction of $\cV$, any two points in $\cV$ are at least $\epsilon$ away, so the balls do not intersect with each other. Notice that as shown by \eqref{eq:spherical_cap}, we can bound the area measure of these balls, i.e., for any $\w \in \cV$, 
	\[
	\frac{\sigma_{p-1}(\Sp_{\epsilon/2} ^{p-1}(\w))}{\sigma_{p-1}(\Sp ^{p-1})} \ge \frac{v_{p-1}}{p v_p} \sin ^{p-1}\frac{\epsilon}{2},
	\]	
	the result thus follows by that
	\[
	\mathcal{C}(\Sp^{p-1}, \epsilon) \le \card(\cV) \le \frac{\sigma_{p-1}(\Sp ^{p-1})}{\sigma_{p-1}(\Sp_{\epsilon/2} ^{p-1}(\w))}
	\]
\end{proof}

\subsection{Proof for Theorem \ref{thm:covering-radius-bound}}

\begin{theorem*}
	Let $\cP \subseteq \Sp^{p-1}, p \ge 2$ be a set of $K$ points that are drawn independently and uniformly at random on $\Sp^{p-1}$. Then for any $\gamma^* \le \pi/2$, it has $\gamma(\pm \cP) < \gamma^*$ with probability at least $1 - \frac{p v_p}{v_{p-1}}\frac{1}{\sin ^{p-1} \frac{\gamma^*}{4}} \exp (-K \frac{2 v_{p-1}}{p v_p}\sin ^{p-1} \frac{\gamma^*}{2})$
\end{theorem*}
\begin{proof}
	Let $\epsilon = \gamma^* / 2$, and let $\cV$ be any $\epsilon$-covering of $\Sp^{p-1}$ such that $\card(\cV) = \mathcal{C}(\Sp^{p-1}, \epsilon)$. Centered at each point of $\cV$ draw a ball with radius of $\epsilon$, then the union of these balls covers the entire sphere. The idea of the proof is that if each of the balls contain at least one point in the set $\pm \cV$, then the covering radius $\gamma(\pm \cV)$ is bounded by $2\epsilon$. This is because that for any $\w \in \Sp^{p-1}$, it lies in at least one of the balls, and when this ball contains at least one point in $\pm \cV$, then the distance $s(\w, \pm \cV)$ is bounded above by $2 \epsilon$. Concretely, denote $M := \card(\cV)$ and let $B_1, \cdots, B_M$ be the balls illustrated above, then
	\[\begin{split}
	P(\gamma > 2\epsilon)
	&\le P(\exists i\in\{1, \cdots, M\} \text{ s.t. } B_i \cap \pm \cP = \emptyset)\\
	&\le \sum_{i=1} ^M P(B_i \cap \pm \cP = \emptyset)\\
	& = \sum_{i=1} ^M (1 - 2\frac{\sigma_{p-1}(B_i)}{\sigma_{p-1}(\Sp^{p-1})}) ^K,\\
	\end{split}\]
	where the factor of $2$ appears in the last line because we are using symmetrized points $\pm \cP$. Notice that each $B_i$ is a spherical cap of radius $\epsilon$, we can use the result of \eqref{eq:spherical_cap} to give a bound on it. We get
	\begin{multline}
	P(\gamma > 2\epsilon)
	\le \sum_{i=1} ^M (1 - \frac{2 v_{p-1}}{p v_p} \sin ^{p-1} \epsilon) ^K\\
	\le M \exp( -K \frac{2 v_{p-1}}{p v_p} \sin ^{p-1} \epsilon),
	\end{multline}
	in which $M$ can be further bounded by result of Lemma \ref{thm:covering-number}, so
	\[
	P(\gamma > 2\epsilon) \le \frac{p}{\frac{v_{p-1}}{v_p}\sin^{p-1}\frac{\epsilon}{2}} \exp( -K \frac{2 v_{p-1}}{p v_p} \sin ^{p-1} \epsilon).
	\]
	This proves the theorem.
\end{proof}

\subsection{Proof for Theorem \ref{thm:randomized}}

In this section, we finish what is undiscussed in the roadmap of proof for Theorem  \ref{thm:randomized} and this will complete the proof.

\begin{proof}

The proof is by giving probabilistic bounds on both sides of DRC separately and then apply the well known union bound to combine the results. In this proof we write $d:= d_0$, $s := s_0$ and $k := k_0$ to simplify notations.

For any $\gamma^* \le \frac{\pi}{2}$, the LHS and RHS of DRC are bounded in \eqref{eq:bound-on-left} and \eqref{eq:bound-on-right}, respectively. By applying union bound we get

\begin{multline}
P(\text{DSC is satisfied}) = P(\gamma_0 < s(\cD_0, \cA_c))\\
\ge 1 - \boxed{\frac{d v_d}{v_{d-1}} \cdot\frac{1}{\sin^{d-1}\frac{\gamma^*}{4}}}\cdot \exp (-\boxed{s \frac{2 v_{d-1}}{d v_d}\sin^{d-1} \frac{\gamma^*}{2}}) \\ \boxed{-\lambda s \cdot \binom{s}{d} 2^d \cdot \frac{v_{D-1}}{v_D} \sin ^{D-1}\gamma^*}.
\label{eq:prop-bound-DSC-1}
\end{multline}

Now, we take a special value of $\gamma^*$ as
\begin{equation}
\sin^{D-1}\gamma^* = (\frac{s}{d})^{-0.5\frac{D-1}{d-1}} \frac{v_D}{v_{D-1}},
\label{eq:proof-randomized-gamma0}
\end{equation}
or equivalently, 
\begin{equation}
\sin^{d-1}\gamma^* = (\frac{s}{d})^{-0.5} \big( \frac{v_D}{v_{D-1}} \big) ^\frac{d-1}{D-1},
\end{equation}
and we will argue that such a $\gamma \le \frac{\pi}{2}$ exists at the end of this proof. 

Define the following for later use:
\begin{equation}
C(D, d) = \frac{1}{2^{d-2}} \frac{v_{d-1}}{v_d} \big( \frac{v_D}{v_{D-1}} \big) ^\frac{d-1}{D-1}.
\end{equation}

For easier presentation, we take three boxed parts from the RHS of \eqref{eq:prop-bound-DSC-1} and provide bounds for them separately, and then combine them to get the final result.

For the first part, we compute
\begin{multline}
\frac{d v_d}{v_{d-1}} \cdot\frac{1}{\sin^{d-1}\frac{\gamma^*}{4}} \le d \cdot 2^{d-1} \frac{v_d}{v_{d-1}} \cdot\frac{2^{d-1}}{\sin^{d-1}\gamma^*}\\
= d \cdot 2^{d-1} \frac{v_d}{v_{d-1}} \cdot \sqrt{ \frac{s}{d} } 2^{d-1} \big( \frac{v_{D-1}}{v_D} \big)^\frac{d-1}{D-1} = \frac{d \cdot 2^d}{C(D,d)} \sqrt{ \frac{s}{d} }.
\end{multline}
in which we have used the result that $\sin(2x) \le 2\sin(x)$ for any $x \in [0, \pi]$.

For the second part,
\begin{equation}
\begin{split}
&s \frac{2 v_{d-1}}{d v_d}\sin^{d-1} \frac{\gamma^*}{2} \ge s \frac{2 v_{d-1}}{d v_d} \frac{\sin^{d-1} \gamma^*}{2^{d-1}} \\
= &s \frac{2 v_{d-1}}{d v_d} \frac{1}{2^{d-1}} \sqrt{ \frac{d}{s} } \big( \frac{v_D}{v_{D-1}} \big) ^\frac{d-1}{D-1} \\
= &\frac{2 v_{d-1}}{d v_d} \frac{d}{2^{d-1}} \sqrt{ \frac{s}{d} } \big( \frac{v_D}{v_{D-1}} \big) ^\frac{d-1}{D-1} = C(D,d) \sqrt{ \frac{s}{d} }.
\end{split}
\end{equation}

For the third part, use the fact that $\binom{s}{d} \le (\frac{e s}{d})^d$, we have

\begin{equation}
\begin{split}
&\lambda s \cdot \binom{s}{d} 2^d \cdot \frac{v_{D-1}}{v_D} \sin ^{D-1}\gamma^* \\
\le &\lambda s \cdot \big( \frac{2e s}{d} \big)^d \cdot \big( \frac{s}{d}\big) ^{-0.5\frac{D-1}{d-1}}\\
= &\lambda d (2e)^d \big( \frac{s}{d}\big) ^{(d+1)-0.5\frac{D-1}{d-1}} \le \lambda d (2e)^d \big( \frac{s}{d}\big) ^{-k}.
\end{split}
\end{equation}

Combining the above three parts into  \eqref{eq:prop-bound-DSC-1} we get

\begin{multline}
P(\text{DSC is satisfied})
\ge 1 - \lambda d (2e)^d \big( \frac{s}{d}\big) ^{-k} \\ - \frac{d \cdot 2^d}{C(D,d)} \sqrt{ \frac{s}{d} } \exp \big(C(D,d) \sqrt{ \frac{s}{d} } \big),
\label{eq:prop-bound-DSC-2}
\end{multline}
which is the conclusion in \eqref{eq:randomized}.

For the rest part of the proof, we will be needing the following result:

\begin{equation}
\frac{v_{p-1}}{v_p} \in \Big[ \frac{p+1}{\sqrt{2\pi(p+2)}}, \sqrt{\frac{p+1}{2\pi}} \Big],
\label{eq:proof-randomized-ratio-v}
\end{equation}
which is acquired by combining the calculation formula of $v_p$ in \eqref{eq:v_p} and the following result \cite{Foucart:13}:
\begin{equation}
\frac{p}{\sqrt{p+1}} \le \sqrt{2} \frac{\Gamma (\frac{p+1}{2})}{\Gamma (\frac{p}{2})} \le \sqrt{p}.
\end{equation}

We now show that the $\gamma^*$ in \eqref{eq:proof-randomized-gamma0} is well-defined. It boils down to showing that the RHS of \eqref{eq:proof-randomized-gamma0} is less than $1$. Note the first factor is less than or equal to one since $s \ge d$. The second factor can be upper bounded by \eqref{eq:proof-randomized-ratio-v}, i.e.

\begin{equation}
\frac{v_D}{v_{D-1}} \le \frac{\sqrt{2\pi(D+2)}}{D+1},
\end{equation}
in which the RHS is a decreasing function in $D$ and is less than $1$ when $D = 7$. As it is required in the theorem that $D > 2 d^2 \ge 8$, we can conclude that the RHS of  \eqref{eq:proof-randomized-gamma0} is less than one.

In the rest part of this proof, we show the properties $C(D, d)$ as a function of $D$ and $d$. First, we show that $C(D, d)$ is increasing in $D$. Compute that

\begin{equation}
\begin{split}
\frac{C(D,d)}{C(D-1,d)} &= \big(\frac{v_D}{v_{D-1}} \big) ^{\frac{d-1}{D-1}}
\cdot \big(\frac{v_{D-2}}{v_{D-1}} \big) ^{\frac{d-1}{D-2}}\\
&\ge \big(  \sqrt{2\pi} \frac{\sqrt{D+2}}{D+1} \big) ^{\frac{d-1}{D-1}} \cdot
\big(  \sqrt{\frac{D}{2\pi}} \big) ^{\frac{d-1}{D-2}}\\
&= \sqrt{\frac{(D+2)D}{(D+1)^2}} ^{\frac{d-1}{D-1}} \cdot \sqrt{\frac{D}{2\pi}} ^{\frac{d-1}{(D-1)(D-2)}}\\
&= \Bigg( \big( \frac{(D+2)D}{(D+1)^2} \big) ^{D-2} \cdot \frac{D}{2\pi} \Bigg)^{\frac{d-1}{2(D-1)(D-2)}}\\
&>1,
\end{split}
\end{equation}
where we have used the result \eqref{eq:proof-randomized-ratio-v}, and the last inequality comes from the following observations: Let $f(D) = \big( \frac{(D+2)D}{(D+1)^2} \big) ^{D-2} \cdot \frac{D}{2\pi}$. One can compute that $f(7) > 1$, and $f$ is an increasing function of $D$ by calculus. Thus, $C(D,d) > C(D-1, d)$ if $D \ge 7$.

Similarly, for showing that $C(D, d)$ is decreasing in $d$, we compute the ratio

\begin{equation}
\begin{split}
&\frac{C(D,d)}{C(D,d-1)} = \frac{1}{2} \frac{v_{d-1}}{v_{d}} \frac{v_{d-1}}{v_{d-2}} \sqrt[(D-1)]{\frac{v_D}{v_{D-1}}}\\
\le& \frac{1}{2} \sqrt{\frac{1}{2\pi}} \sqrt{d+1} \sqrt{2\pi} \frac{\sqrt{d+1}}{d} \sqrt[(D-1)]{\sqrt{2\pi}\frac{\sqrt{D+2}}{D+1}}\\
=& \frac{d+1}{2d} \sqrt[(D-1)]{\sqrt{2\pi}\frac{\sqrt{D+2}}{D+1}} < 1,
\end{split}
\end{equation}
in which we have used the result \eqref{eq:proof-randomized-ratio-v}, and in the last step we use the fact that $\frac{d+1}{2d}\le 1$ when $d \ge 2$, and that $\sqrt{2\pi}\frac{\sqrt{D+2}}{D+1} < 1$ when $D \ge 7$.

Finally, to give a lower bound on $C(D, d)$, we use equation \eqref{eq:proof-randomized-ratio-v} again and get 
\begin{equation}
C(D, d) \ge \frac{1}{2^{d-2}} \frac{d+1}{\sqrt{2\pi(d+2)}} \Big( \sqrt{\frac{2\pi}{D+1}} \Big) ^\frac{d-1}{D-1}.
\end{equation}

For the RHS, we can have the bound $\frac{d+1}{\sqrt{d+2}} > \sqrt{d}$. Moreover, let $g(D)= (\frac{2\pi}{D+1})^\frac{0.5}{D-1}$, by calculus, one can see that $g(D)$ takes minimum when $D=14$. Thus
\begin{equation}
C(D, d) \ge \sqrt{\frac{2}{\pi}} \sqrt{d} \Big( \frac{g(14)}{2} \Big) ^{d-1} > \frac{0.79 \sqrt{d}}{2.07^{d-1}}.
\end{equation}

This finishes all the claims of the theorem.
\end{proof}

\section{Proof of results in section \ref{sec:sparse-recovery}}

\subsection{Proof of Lemma \ref{thm:sparse-recovery-dualpoint}}
\begin{lemma*}
	For $\cA_0$ which has $s_0$ linearly independent atoms, the set of dual points, $\cD_0$, contains exactly $2^{s_0}$ points specified by $\{\A_0 (\A_0^\transpose \A_0) ^{-1} \cdot \u, \u \in U_{s_0}\}$, where $U_{s_0} := \{ [u_1, \cdots, u_{s_0}], u_i = \pm 1, i = 1, \cdots, s_0 \}$.
\end{lemma*}
\begin{proof}
	From Lemma \ref{thm:dual-finite}, there are possibly at most $2^s$ dual points in the case where $\A_0$ is of full column rank. So in order to prove the result, it is enough to show that the set $\{\A_0 (\A_0^\transpose \A_0) ^{-1} \cdot \u, \u \in U_s\}$ contains $2^s$ points, and each of them is a dual point. 
	
	To show that there are $2^s$ different points, notice that $U_s$ has $2^s$ points, so we are left to show that for any $\u_1, \u_2 \in U_s$ with $\u_1 \ne \u_2$, it has $\A_0 (\A_0^\transpose \A_0) ^{-1} \u_1 \ne \A_0 (\A_0^\transpose \A_0) ^{-1} \u_2$. This can be easily established by noticing that $\rank(\A_0 (\A_0^\transpose \A_0) ^{-1}) = \rank(\A_0) = s$, i.e., $\A_0 (\A_0^\transpose \A_0) ^{-1}$ is also of full column rank, so its null space contains only the origin. Consequently, if $\A_0 (\A_0^\transpose \A_0) ^{-1} \u_1 = \A_0 (\A_0^\transpose \A_0) ^{-1} \u_2$, then $\u_1 = \u_2$, which is a contradiction.
	
	Now we show that $\A_0 (\A_0^\transpose \A_0) ^{-1} \u_0$ is a dual point for any $\u_0 \in U_s$. Denote $\v_0 = \A_0 (\A_0^\transpose \A_0) ^{-1} \u_0$. By definition, we need to show that $\v_0$ is an extreme point of the set $\cK_0^o = \{\v \in \cS_0: \|\A_0^\transpose \v\|_\infty \le 1\}$. First, $\v_0$ is in $\cK_0^o$ because $\|\A_0 ^\transpose \v_0\|_\infty = \|\u_0\|_\infty = 1$. Second, suppose there are two points, $\v_1, \v_2 \in \cK_0^o$, such that 
	\begin{equation}
	\v_0 = (1-\lambda) \v_1 + \lambda \v_2
	\label{eq:proof-convex-cmb}
	\end{equation}
	for some $\lambda \in (0, 1)$, we need to show that it must be the case that $\v_1 = \v_2$. Notice that the columns of $\A_0 (\A_0^\transpose \A_0) ^{-1}$ span the space $\cS_0$ and that $\v_1, \v_2 \in \cK_0^o \subseteq \cS_0$, there exists $\x_1, \x_2$ such that $\v_i = \A_0 (\A_0^\transpose \A_0) ^{-1} \x_i, i = 1, 2$. Then by using \eqref{eq:proof-convex-cmb}, it has
	\begin{multline}
	\A_0 (\A_0^\transpose \A_0) ^{-1} \u_0 \\= (1-\lambda) \A_0 (\A_0^\transpose \A_0) ^{-1} \x_1 + \lambda \A_0 (\A_0^\transpose \A_0) ^{-1} \x_2,
	\end{multline}
	and by left multiplying $\A_0 ^\transpose$, we have \begin{equation}
	\u_0 = (1-\lambda) \x_1 + \lambda \x_2.
	\label{eq:proof-convex-cmb-coeff}
	\end{equation}
	Now, consider the equation for each entry separately in \eqref{eq:proof-convex-cmb-coeff}, i.e., $[\u_0]_i = (1-\lambda) [\x_1]_i + \lambda [\x_2]_i$, where $i$ indexes an entry in the vector. The left hand side, being $\pm 1$, is a extreme point of the set $[-1, 1]$, while the right hand side is the convex combination of two points in $[-1, 1]$, so it necessarily has that $[\x_1]_i = [\x_2]_i$. This is true for all entries $i$, so $\x_1 = \x_2$, thus $\v_1 = \v_2$, which shows that $\v_0$ is indeed an extreme point.
\end{proof}

\subsection{Proof of Theorem \ref{thm:sparse-recovery-compare}}

\begin{theorem*}
	If a dictionary $\cA$ satisfies $\mu(\cA) < \frac{1}{2{s_0}-1}$, then for any partition of $\cA$ into $\cA_0$ and $\cA_c$ where $\card(\cA_0) = s_0$, it has that the atoms in $\cA_0$ are linearly independent and that PRC and DRC hold.
\end{theorem*}
\begin{proof}
	Suppose $\mu(\cA) < 1 / (2s-1)$, we need to show that $\rank(\A_0) = s$ and that PRC and DRC holds. First, the result that $\rank(\A_0) = s$ is well established in studies of sparse recovery. We then only need to show that PRC is true, as DRC is implied by PRC.
	
	We start by giving an upper bound on $1/\cos \gamma_0$. From Lemma \ref{thm:sparse-recovery-dualpoint}, given any $\v \in \cK_0^o$ where $\v \ne 0$, it can be written as $\v = \A_0 (\A_0^\transpose \A_0) ^{-1} \u$ for some $\u \ne 0$ with $\|\u\|_\infty \le 1$. Thus, 
	\[
	\|\v\|_2 ^2 = \v^\transpose \v = \u^\transpose (\A_0^\transpose \A_0)^{-1} \u \le s \cdot \frac{\u^\transpose (\A_0^\transpose \A_0)^{-1} \u}{\u^\transpose \u}.
	\]
	Denote $\lambda_{\max} (\cdot), \lambda_{\min} (\cdot)$ to be the maximum and minimum eigenvalue of a symmetric matrix, respectively. We get
	\[\begin{split}
	\|\v\|_2 ^2 &\le s \cdot \max_{\u \ne 0} \frac{\u^\transpose (\A_0^\transpose \A_0)^{-1} \u}{\u^\transpose \u}\\
	& = s \cdot \lambda_{\max} (\A_0^\transpose \A_0)^{-1} = 
	\frac{s}{\lambda_{\min} (\A_0^\transpose \A_0)}.
	\end{split}	\]
	Notice that $\A_0 ^\transpose \A_0$ is close to an identity matrix, i.e., its diagonals are $1$ and the magnitude of each off-diagonal entry is bounded above by $\mu(\cA)$. By using Gersgorin's disc theorem, $\lambda_{\min} (\A_0^\transpose \A_0) \ge 1 - (s-1)\mu(\cA)$, so
	\[
	\|\v\|_2 ^2 \le \frac{s}{1 - (s-1)\mu(\cA)}.
	\]
	As a consequence, $1/\cos \gamma_0 \le \sqrt{\frac{s}{1 - (s-1)\mu(\cA)}}$ by Lemma \ref{thm:bound-polar}.
	
	In the second step, we give an upper bound for the right hand side of PRC. By definition,
	\[
	\cos s(\cA_c, \cS_0) = \max_{\substack{\v \in \cS_0,\\\|\v\|_2 = 1}} \|\A_c ^\transpose \v\|_\infty.
	\]
	We thus need to bound $\|\A_c ^\transpose \v\|_\infty$ for any $\v \in \cS_0$ with $\|\v\|_2 = 1$. Consider the optimization program
	\[
	\x^* = \arg\min_{\x} \|\x\|_1 \st \v = \A_0 \x.
	\]	
	and its dual program
	\[
	\max_\omega \langle \omega, \v \rangle \st \|\A_0 ^\transpose \omega\|_{\infty} \le 1.
	\]
	The strong duality holds since the primal problem is feasible, and the objective of the dual is bounded by $\|\omega\|_2 \|\v\|_2 \le 1/\cos \gamma_0$. Consequently, it has $\|\x^*\|_1 \le 1/\cos \gamma_0$. This leads to
	\[\begin{split}
	\|\A_c ^\transpose \v\|_\infty &= \|\A_c ^\transpose \A_0 \x^*\|_\infty \le \|\A_c ^\transpose \A_0\|_\infty \|\x^*\|_1 \\
	&\le \mu(\cA) /\cos \gamma_0,
	\end{split}\]
	in which $\|\cdot\|_\infty$ for matrix treats the matrix as a vector.
	
	Now we combine the results from the above two parts.
	\[\begin{split}
	\cos s(\cA_c, \cS_0) &\le \mu(\cA) /\cos \gamma_0 \\
	&= \cos \gamma_0 \cdot (\mu(\cA) /\cos \gamma_0^2)\\
	&\le \cos \gamma_0 \frac{s \mu(\cA)}{1 - (s-1)\mu(\cA)},
	\end{split}\]
	in which
	\[
	\frac{s \mu(\cA)}{1 - (s-1)\mu(\cA)} = 1 + \frac{\mu(\cA)(2s-1) - 1}{1-(s-1)\mu} < 1,
	\]
	thus $\cos s(\cA_c, \cS_0) < \cos \gamma_0$, which is the PRC.
\end{proof}

\section{Proof of results in section \ref{sec:sparse-classification}}

Theorem \ref{thm:sparse-classification-deterministic} is a trivial application of the result in theorem \ref{thm:DRC} to all subspaces $i=1, \cdots, n$. 

For Theorem \ref{thm:sparse-classification-randomized}, the result is acquired by applying union bound. We give more details on this proof since the probabilistic model is not the same as that in Theorem \ref{thm:randomized} and there are certain points that need to be explained and clarified. 

Concretely, let $E_i$ be event that the condition 
\begin{equation}
	\gamma_i < s(\cD_i, \cA \backslash \cA_i)
	\label{eq:proof-sparse-classification}
\end{equation}
is satisfied, $i = 1, \cdots, n$. For a fixed $i$, the LHS of \eqref{eq:proof-sparse-classification} can be upper bounded in the same way as in \eqref{eq:bound-on-left} by using Theorem \ref{thm:covering-radius-bound}, i.e.
\begin{multline}
P(\gamma_i < \gamma^*) \ge 1 -  \frac{d_i \cdot v_{d_i}}{v_{(d_i-1)}} \cdot\frac{1}{\sin^{(d_i-1)}\frac{\gamma^*}{4}} \\ \cdot \exp (-s_i \frac{2 v_{(d_i-1)}}{v_{d_i}}\sin^{(d_i-1)} \frac{\gamma^*}{2}).
\end{multline}

For the RHS of \eqref{eq:proof-sparse-classification}, the analysis is similar to that that leads to Equation \eqref{eq:bound-on-right}. For any point $\v \in \cD_i$ and $\w \in \cA \backslash \cA_i$, we observe that both of them have a uniform distribution on the unit sphere $\Sp^{D-1}$, and that they are independent due to the fact that they are from different subspaces. Thus one gets

\begin{equation}
P(s(\cD_i,  \cA \backslash \cA_i) > \gamma^*) \ge 1 - (s- s_i) \cdot \binom{s_i}{d_i} 2^{d_i} \cdot \frac{v_{D-1}}{v_D} \sin ^{D-1}\gamma^*.
\end{equation}

By combining these two bounds in the same way as in the proof of Theorem \ref{thm:randomized}, one get

\begin{multline}
	P(E_i) \ge 1 -\frac{d_i\cdot2^{d_i}}{C(D,d_i)} \sqrt{\rho_i} \cdot e^{-C(D, d_i) \sqrt{\rho_i}} - \frac{\sum_{j\ne i} s_j}{s_i}\frac{d_i (2e)^{d_i}}{{(\rho_i)} ^{k_i}} \\
	> 1 -\frac{d_i\cdot2^{d_i}}{C(D,d_i)} \sqrt{\rho_i} \cdot e^{-C(D, d_i) \sqrt{\rho_i}} -\frac{d_i (2e)^{d_i}}{p_i (\rho_i) ^{k_i}}.
\end{multline}
 
By applying union bound,
\begin{equation}
	P(\text{SRC succeeds}) = P(\cap_{i=1}^n E_i) \ge 1 - \sum_{i=1}^n (1 - P(E_i)),
\end{equation}
one can get the conclusion in \eqref{eq:sparse-classification-randomized}.

\section*{Acknowledgment}

The authors would like to thank the support of NSF BIGDATA grant 1447822.

\ifCLASSOPTIONcaptionsoff
  \newpage
\fi



\bibliographystyle{IEEEtran}
\bibliography{IEEEabrv,biblio/vidal,biblio/vision,biblio/math,biblio/learning,biblio/sparse,biblio/geometry,biblio/dti,biblio/recognition,biblio/surgery,biblio/coding,biblio/matrixcompletion,biblio/segmentation}
\end{document}